\documentclass{article}
\usepackage{ijcai22}

\usepackage{times}
\usepackage{soul}
\usepackage{url}
\usepackage[hidelinks]{hyperref}
\usepackage[utf8]{inputenc}
\usepackage[small]{caption}
\usepackage{graphicx}
\usepackage{amsmath}
\usepackage{amsthm}
\usepackage{booktabs}
\usepackage{algorithm}
\usepackage{algorithmic}
\usepackage{multicol}

\usepackage{amssymb}
\usepackage{enumerate}
\usepackage{tikz}
\usetikzlibrary{patterns,backgrounds}
\usepackage{xspace}
 
\usepackage{fancyvrb}
\usepackage{fvextra}
\usepackage{relsize}
\fvset{fontfamily=courier,fontsize=\relsize{-1},frame=none,numbers=none,numbersep=4pt,commentchar=!,commandchars=\\\{\}}
\newcommand\atom{\Verb*[fontfamily=courier,fontsize=\relsize{0}]}

\newcommand{\Omit}[1]{}

\newcommand{\tup}[1]{\langle #1 \rangle}

\newcommand{\pair}[1]{\langle #1 \rangle}

\newcommand{\citeay}[1]{\citeauthor{#1} [\citeyear{#1}]}

\newcommand{\eqdef}{\stackrel{\textrm{\scriptsize def}}{=}}

\newcommand{\multiset}[1]{\{\!\!\{#1\}\!\!\}}

\renewcommand{\S}{\mathcal{S}}

\newcommand{\clingo}{\textsc{Clingo}\xspace}
\renewcommand{\P}{\mathcal{P}}
\newcommand{\D}{\mathcal{D}}
\newcommand{\T}{\mathcal{T}}

\newtheorem{definition}{Definition}
\newtheorem{theorem}[definition]{Theorem}
\colorlet{acolor}{red!80!black}
\colorlet{pcolor}{green!60!black}
\colorlet{bcolor}{blue!70!black}

\title{Learning First-Order Symbolic Planning Representations That Are Grounded}

\author{
 Andr\'{e}s Occhipinti$^1$\and
 Blai Bonet$^1$\and
 Hector Geffner$^{2,1,3}$\\
 \affiliations
 $^1$Universitat Pompeu Fabra, Barcelona, Spain\\
 $^2$ Instituci\'o Catalana de Recerca i Estudis Avan\c{c}ats (ICREA), Barcelona, Spain\\
 $^3$Link{\"o}ping University, Link\"oping, Sweden\\
 \emails
 andres.occhipinti@upf.edu, bonetblai@gmail.com, hector.geffner@upf.edu}

\begin{document}

\maketitle

%%% abstract
%% intro
%% related work
%% background
%% formulation
%% implementation
%% experimental results
%% analysis
%% conclusions

\begin{abstract}
  Two main  approaches have been developed for learning first-order  planning (action) models from
  unstructured data:  combinatorial  approaches that yield  crisp  action schemas from the structure of
  the state space, and deep learning approaches that produce  action schemas
  from states represented by  images. A benefit of the former approach is that 
  the learned action schemas are   similar to those that can be written by hand;
  a benefit of the latter  is that the learned representations (predicates)
  are grounded on the images, and as  a result, new instances can be given
  in terms of images. In this work, we develop a new formulation for learning crisp first-order planning models
  that are grounded on \emph{parsed images,} a step to combine the benefits of the two approaches.
  Parsed images are assumed to be given in a simple O2D language (objects in 2D)
  that involves a small % and fixed
  number of unary and binary predicates like `left', 'above', 'shape', etc.
  After learning, new  planning instances  can be given in terms of pairs of
  parsed images, one for the initial situation and the other for the  goal.
  Learning and planning experiments are reported for several  domains  including Blocks, Sokoban, IPC Grid, and Hanoi.
\end{abstract}

\section{Introduction}

One of the key open problems in AI is how to combine learning and reasoning, in particular
when the learning data is not structured for reasoning. In  planning, there are effective 
reasoning mechanisms for  planning  but they rely on  models comprised of
predicates and action schemas which are usually provided by hand.
A number of proposals have been advanced for learning and refining these models, 
but most  assume  that the domain predicates are known
 \cite{yang:model-learning,zhuo:learning,petrick:learning,rao:learning,aineto:learning,paolo:learning}.
The problem of {learning the domain predicates and the action schemas
at the same time}  is more challenging.  A clever approach for addressing
this problem given  sequences of grounded actions  was developed in  the LOCM  system
\cite{locm,locm3},  although the approach is heuristic and incomplete. 
Two recent formulations have addressed  the  problem
more systematically and  without assuming that action arguments are observable.
One is a combinatorial approach that yields crisp action schemas from the structure of the state space
\cite{bonet:ecai2020,ivan:kr2021}; the other is a deep learning approach that produces action
schemas from states represented by images \cite{asai:fol}.
A benefit of the combinatorial approach is that it accommodates  and exploits
a natural inductive bias  where fewer, simpler action schemas and predicates are preferred; 
a bias that results in learned action schemas that are similar to those that can be
written by hand. A benefit of deep learning approaches is that the learned representations (predicates)
are grounded on the images, and as  a result, new instances can be given in terms of them.

The aim of this work is to develop a new formulation for learning crisp first-order planning
models that are grounded on \emph{parsed images}, a step in the way to combine the benefits of the
combinatorial and deep learning approaches. Parsed images are assumed to be given in a simple O2D language, for ``objects in 2D'', that
involves a small number of unary and binary ``visual''  predicates like `left',
'above', 'shape', etc. The learning problem becomes the problem of learning the lifted domain
(action schemas and domain predicates) along with the \emph{grounding} of the learned domain
predicates so that they can be evaluated on any parsed image.\footnote{Grounding a predicate
  (symbol) means to provide a semantics for it in the form of a denotation function, and 
   should not be confused with  grounding of the action schemas.
  % that refers to replacing   the argument variables by constants. %  and thus is a form of grounding the schema variables.
  For the problem of  grounding symbols in the ``real world'',  see \citeay{symbol-grounding}.
}
A number of vision modules can  be used to map images into the parsed representations
\cite{yolo,yolo2,k-slot-attention},
% or one could deal in an integrated way with the full pipeline
but this is outside the scope of this work.
\Omit{
  \footnote{%
  When using combinatorial approaches for learning, a question that often comes out is how
  to deal with noise. There is indeed a difference between deep learning approaches that
  can accept noisy inputs and tend to produce noisy outputs, and combinatorial approaches
  that can produce  crisp outputs but which  often require noise-free inputs.
  Yet, it's not difficult to make the latter approaches tolerant to noise and incompleteness
  by given the solver the ability to label certain parts of the inputs as noisy or incomplete, at
  a suitable cost \cite{ivan:kr2021}.
  }
  }
    After learning, new planning instances can be given in terms of pairs of parsed images, one corresponding
to the initial situation and the other to the goal, and such instances can be solved 
with  any off-the-shelf planner  and may  involve many more objects than those used in training.
Learning and planning results for several domains are reported,
including Blocks, Towers of  Hanoi, the $n$-sliding-puzzle, IPC Grid, and Sokoban.
% Interestingly, in comparison with approaches that learn first-order planning representation from the same inputs that
% are not grounded, by considering the states as black boxes instead, the proposed approach that learns grounded representation
% scales up to more interesting and challenging problems.

The paper is organized as follows. We discuss related work, review the basics of classical
planning, and introduce the O2D language and the learning formulation.
An implementation of the learning formulation as an answer set program is then sketched
(full details in the appendix), and experimental results are presented and analyzed.

\section{Related Work}

Most works on learning action schemas from traces assume that the domain predicates  are known \cite{yang:model-learning,littman:strips,zhuo:learning,petrick:learning,rao:learning,juba:learning,aineto:learning,paolo:learning}. The problem of learning the action schemas and the predicates  at the same time is more challenging
as the structure of the  states is   not available at all. The LOCM systems \cite{locm,locm2,locm3,locm4}   addressed this
problem  assuming input traces that feature sequences of ground actions.
The inference of action schemas and predicates follows  a number of heuristic rules that  manage to learn
challenging planning domains  but whose   soundness and completeness properties  have not been studied.
A general formulation of the learning problem from complete input traces that feature just action names and black-box states
is given by \citeay{bonet:ecai2020}, and extensions for  dealing with incomplete and noisy traces
by \citeay{ivan:kr2021} (see also \citeay{sid:learning}).
An alternative, deep learning approach for learning action schemas and predicates from states represented by
images is advanced by \citeay{asai:fol}. The advantages of a deep learning approach based on
images are several: it does not face the scalability bottleneck of combinatorial approaches, it is robust to noise,
and it yields representations \emph{grounded} in the images.  The limitation is that the learned planning domains
tend to be complex and opaque. For example, \citeauthor{asai:fol} reports 518,400 actions for a Blocksworld instance with 3 blocks.
% This is not a problem in combinatorial approaches that deal with  a strongly biased hypothesis space where the simplest domains
% learned  are in  close correspondence with the domains  written by hand.
Methods for learning \emph{propositional} planning representations that are grounded have also
been proposed \cite{konidaris:jair,asai:prop1,asai:prop2} but they are bound to work in a single state space involving a fixed set of objects.

\section{Classical Planning}

A (classical) planning instance is a pair $P\,{=}\,\tup{D,I}$ where $D$ is
a first-order planning domain and $I$ represents {instance information}
\cite{geffner:book,ghallab:book,book:pddl}.
The planning domain $D$ contains a set of predicates (predicate symbols)
$p$ and a set of action schemas with preconditions and effects given by
atoms $p(x_1, \ldots, x_k)$ or their negations, where $p$ is a domain
predicate and each $x_i$ is a variable representing one of the arguments
of the action schema.
The instance information is a tuple $I\,{=}\,\tup{O,Init,Goal}$ where $O$
is a (finite) set of objects (object names) $o_i$, and $Init$ and $Goal$ are
sets of ground atoms $p(o_1, \ldots, o_k)$ or their negations, with $Init$
being consistent and complete; i.e., for each ground atom $p(o_1, \ldots, o_k)$,
either the atom or its negation is (true) in $Init$.
The set of all ground atoms in $P\,{=}\,\tup{D,I}$, $At(P)$, is given by all
the atoms that can be formed from the predicates in $D$ and the objects in $I$,
while the set of ground actions $A(P)$ is given by the action schemas with 
their arguments  replaced by objects in $P$.
A state $s$ over $P$ is a maximally consistent set of ground literals
representing a truth valuation over the atoms in $At(P)$, and a ground action
$a \in A(P)$ is applicable in $s$, written $a \in A(s)$ when its preconditions
are (true) in $s$. A state $s'$ is the successor of ground action $a$ in state
$s$, written $s'\,{=}\,f(a,s)$ for $a \in s$ if the effects of $a$ are true in $s'$
and the truth of atoms not affected by $a$ is the same in $s$ and $s'$.
Finally, an action sequence $a_0, \ldots, a_{n}$ is a plan for $P$ if there is
a  state sequence $s_0, \ldots, s_{n+1}$ such that $s_0$ satisfies $Init$, $s_{n+1}$
satisfies  $Goal$, $a_i \in A(a_i)$, and $s_{i+1}=f(a_i,s_i)$. 

% Planning, $D$, state transition function, ground actions, ... No need to talk about
% graphs .. Talk about $F_a(s)$. Multiset (Unlike sets, multisets allow for multiple
% occurrences of each element contained).

\section{Language of Parsed Images: O2D}

\begin{figure}[t]
  \centering
  \resizebox{\columnwidth}{!}{
    \footnotesize
    \begin{tabular}{@{}c@{\ }c@{\ }c@{\ }c@{\ }c@{}}
      \resizebox{0.200\linewidth}{!}{
        \begin{tikzpicture}[every node/.style={minimum size=0.5cm-\pgflinewidth, outer sep=0pt}, background rectangle/.style={fill=gray!45}, show background rectangle]
          \draw[step=0.1cm,color=white] (0,0) grid (1.5,1.5);
          \node[fill=green!50!black] at (0.55,0.75) {};
          \node[fill=green!90!black] at (1.05,0.75) {};
        \end{tikzpicture}
      } &
      \resizebox{0.200\linewidth}{!}{
        \begin{tikzpicture}[every node/.style={minimum size=0.5cm-\pgflinewidth, outer sep=0pt},, background rectangle/.style={fill=gray!45}, show background rectangle]
          \draw[step=0.1cm,color=white] (0,0) grid (1.5,1.5);
          \node[fill=green!50!black] at (0.75,0.55) {};
          \node[fill=green!90!black] at (0.75,1.05) {};
        \end{tikzpicture}
      } &
      \resizebox{0.200\linewidth}{!}{
        \begin{tikzpicture}[every node/.style={minimum size=0.5cm-\pgflinewidth, outer sep=-2pt},, background rectangle/.style={fill=gray!45}, show background rectangle]
          \draw[step=0.1cm,color=white] (0,0) grid (1.5,1.5);
          \node[fill=green!50!black] at (0.65,0.85) {};
          \node[fill=green!90!black] at (0.85,0.65) {};
        \end{tikzpicture}
      } &
      \resizebox{0.200\linewidth}{!}{
        \begin{tikzpicture}[every node/.style={minimum size=0.5cm-\pgflinewidth, outer sep=0pt},, background rectangle/.style={fill=gray!45}, show background rectangle]
          \draw[step=0.1cm,color=white] (0,0) grid (1.5,1.5);
          \node[fill=green!50!black,scale=0.58] at (0.45,0.75) {};
          \node[fill=green!90!black] at (0.95,0.75) {};
        \end{tikzpicture}
      } &
      \resizebox{0.200\linewidth}{!}{
        \begin{tikzpicture}[every node/.style={minimum size=0.27mm, inner sep=0mm, outer sep=0pt},, background rectangle/.style={fill=gray!45}, show background rectangle]
          \draw[step=0.1cm,color=white] (0,0) grid (1.5,1.5);
          \node[fill=green!90!black,scale=0.80] at (0.55,0.55) {\large$\bigcirc$};
          \node[fill=green!90!black,scale=0.870] at (0.95,0.95) {\Huge$\star$};
        \end{tikzpicture}
      } \\[-.15em]
      \textit{left} & \textit{below} & \textit{overlap} & \textit{smaller} & \textit{shape}
    \end{tabular}
  }
  \vskip -.4em
  \caption{Depiction of the five binary relations in O2D.}
  \label{fig:o2d:rels}
\end{figure}
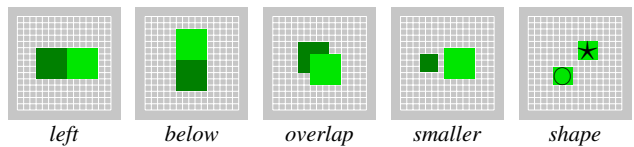

Object-recognition vision systems typically map images into object-tuples of the
form $\{\tup{type(c),loc(c), bb(c),att(c)}\}_c$ that encode the different objects
$c$ in the scene, their type or class, their location and bounding box coordinates,
and some visual attributes like color or shape \cite{yolo,yolo2,k-slot-attention}.
We use a similar encoding of scenes but rather than representing the exact locations
of objects, spatial relations are represented \emph{qualitatively} \cite{qualitative-spatial}.
More precisely, a scene is represented by a set of ground atoms over a language that
we call O2D for \textit{Objects in 2D space}. O2D is a first-order language with signature $\Sigma=(C,U,R)$ where $C$ stands for a
set of constant symbols representing objects and shapes, $U$ stands for a set of unary
predicates, and $R$ stands for a set of binary predicates.
%% The set of predicates is \emph{fixed} \textcolor{red}{\bf *** CHECK THIS FOR $U$ ***}.
%%% For the predicates to to be fixed, Types and Colors should be binary predicates
%%%% But this will add more objects to the encodings; types and colors
%%% Eventually to be settled; not ideally, but now we are not saying that
%%% the set of O2D predicates is fixed
The unary predicates in  $U$ stand for visually different object types and characteristics, 
% of objects (static types) and colors (unary dynamic properties),
% and only the former are used.
% \textcolor{red}{(** This is not the case currently. Classes of objects are not static predicates; some actions make objects ``change'' class: e.g., unlocking a cell in Grid ``transforms'' its type from \textit{lockedcell} to \textit{cell}. These are \textit{visually differentiated object types}: lockedcells are black squares with a symbol in the center, open cells are light gray squares.**)}
while the  binary predicates are $R=\{\textit{left},\textit{below},\textit{overlap},\textit{smaller},\textit{shape}\}$,
representing if one object is right to the left of or right below another object,
if two objects overlap, if one object is smaller than another, and the shape of an
object; see Figure~\ref{fig:o2d:rels}.

A scene is represented in O2D as a set of \emph{ground atoms} over the symbols in
$\Sigma=(C,U,R)$. We refer to scene representations in O2D as \emph{O2D states}.
Scenes and their corresponding O2D states for Blocks-world, Tower-of-Hanoi, %% Sliding-puzzle,
and Sokoban are shown in Figure~\ref{fig:o2d:examples}, with renderings obtained
with  PDDLGym \cite{pddlgym}. 
% \textcolor{red}{\bf *** Mention PDDLGym (visual renderings?) ***}
% The domain IPC Grid where an agent has to pick up keys of certain shapes that
% open doors with matching locks where other keys are stored, and so on, will also
% be considered by the figure is omitted. *** Why need open-cell in sliding puzzle?
% For representing blank? Logically required or to decrease action arities? Important.
% Also ``opencell'' not listed in $\Sigma$ and not exactly an static type **

\Omit{
  Figure \ref{fig:blocks} illustrates the relationship between a state image and its
  O2D representation. The left side of the figure shows the image of a Blocksworld
  state $s$, generated with PDDLGym's graphical representation \cite{silver2020pddlgym}.
  The right side lists some of the ground O2D atoms in $\sigma(s)$, for the language
  whose signature $\sigma$ has constants $\mathcal{O}=\{r,t,b_0,\dots,b_6\}$ for the
  robot ($r$), table ($t$) and the blocks ($b_i$), and object types $\mathcal{C}=\{robot, block, table\}$.
  For a state $s$, we denote by $\sigma(s)$ the representation of $s$ given by
  $\mathcal{L}_\sigma$, i.e., the set of all ground literals from $\mathcal{L}_\sigma$
  that are true in $s$.
}
  
\begin{SaveVerbatim}[gobble=0,commandchars=\\\{\}]{o2d:blocks}
  \textcolor{bcolor}{% objects and types}
  robot(r). table(t).
  block(b0). block(b1).
  ...
  \textcolor{bcolor}{% relations}
  overlap(b0,r).
  below(t,b1).
  smaller(b1,r).
  ...
  \textcolor{bcolor}{% shapes}
  shape(r,rectangle).
  shape(b0,rectangle).
  ...
\end{SaveVerbatim}
\begin{SaveVerbatim}[gobble=0,commandchars=\\\{\}]{o2d:hanoi}
  \textcolor{bcolor}{% object and types}
  disk(d1). disk(d2).
  peg(p1). peg(p2).
  ...
  \textcolor{bcolor}{% relations}
  overlap(d1, p1).
  below(d2,d1).
  smaller(d1, d2).
  ...
  \textcolor{bcolor}{% shapes}
  shape(d1,rectangle).
  ...

\end{SaveVerbatim}

\begin{SaveVerbatim}[gobble=0,commandchars=\\\{\}]{o2d:sokoban}
  \textcolor{bcolor}{% object and types}
  sokoban(s).
  crate(c1). crate(c2).
  cell(c1_1).
  ...
  \textcolor{bcolor}{% relations}
  overlap(s,c3_6).
  overlap(c2, c3_5).
  below(c3_6,c_2_6).
  ...
  \textcolor{bcolor}{% shapes}
  shape(c1,rectangle).
  ...
\end{SaveVerbatim}

\begin{figure*}[t]
  \centering
  \begin{tabular}{@{}c@{}c|c@{}c|c@{}c}
    \includegraphics[width=0.14\textwidth]{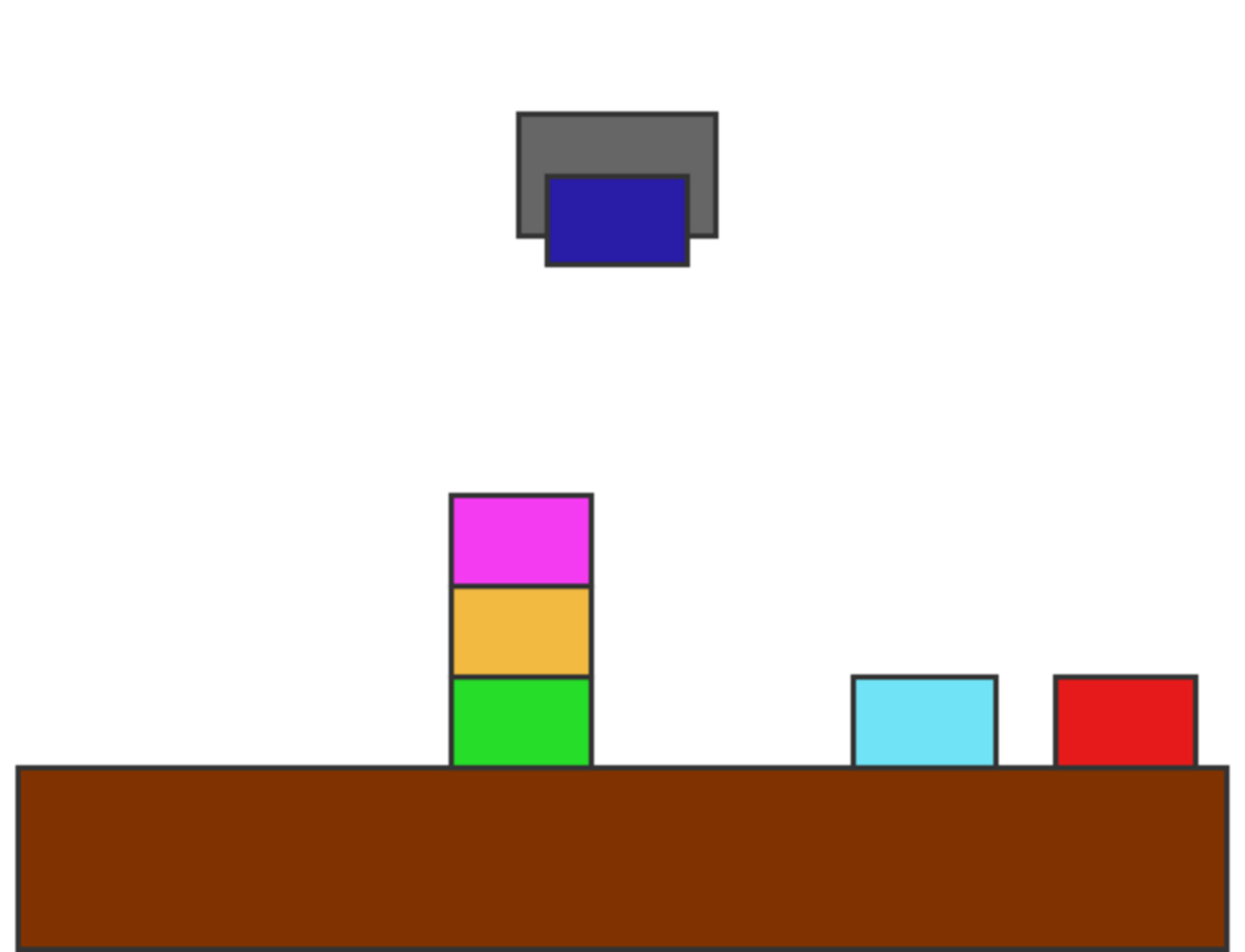} &
    \resizebox{!}{1.3in}{\BUseVerbatim{o2d:blocks}} &
    \includegraphics[width=0.16\textwidth]{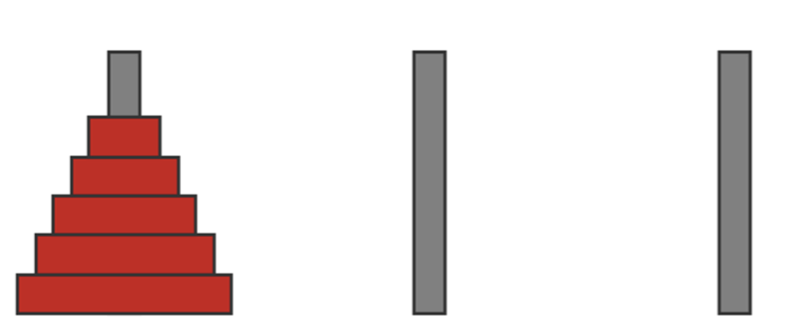} &
    \resizebox{!}{1.3in}{\BUseVerbatim{o2d:hanoi}} &
    %\includegraphics[width=0.125\textwidth]{images/sliding.png} &
    %\resizebox{0.125\textwidth}{!}{\BUseVerbatim{o2d:slidingtile}} &
    \includegraphics[width=0.125\textwidth]{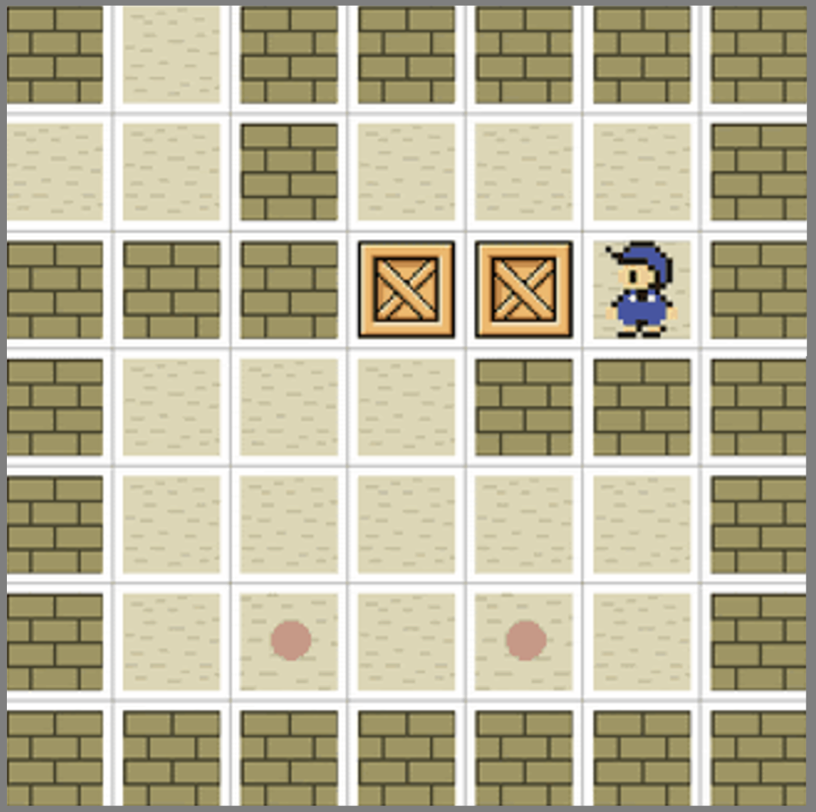} &
    \resizebox{!}{1.3in}{\BUseVerbatim{o2d:sokoban}}
  \end{tabular}
  \vskip -.4em
  \caption{Scenes for three of the domains considered and their O2D representations (see Appendix for details).}
  %\textcolor{red}{(** Need Cell type in Puzzle? Remove opencell? Change opencell in
  % Sokoban to GoalCell type? **}}
  %\caption{\textcolor{red}{\bf *** MISSING CAPTION: EDIT DRAWINGS: opencell not in puzzle. Object *and* types instead of just types ***}}
  \label{fig:o2d:examples}
\end{figure*}

\section{Groundings}

A grounded predicate $q$ is  a predicate that can be evaluated in any O2D state $s$;
i.e., if $o$ is a tuple of objects in $s$ of the same arity as $q$, then $q(o)$ is known to be
true or to be false in $s$.
% Logically, a grounded predicate (symbol) $q$ has a known
% denotation  $q^s$ in every O2D state $s$ given by a boolean function that takes
% tuples of objects in $s$ of the same arity as $q$ as arguments.
The predicates $p$ appearing in a planning domain $D$
are  grounded  by assuming  a pool $\P$ of grounded predicates
and a \emph{grounding function}  $\sigma$ that maps the domain predicates $p$
into grounded predicates $q=\sigma(p)$ in the pool with the same arity as $p$.
The result is a \emph{grounded domain}:

\begin{definition}[Grounded Domain]
  A \emph{grounded domain} over a pool of grounded predicates $\P$
  is a pair $\tup{D,\sigma}$ where $D$ is
  a planning domain $D$ and $\sigma$ is a function that maps each  predicate $p$ in $D$ into a
  predicate $q=\sigma(p)$ in $\P$. 
\end{definition}

The truth value of an atom $p(o)$ in a scene $s$ is the value
of the atom $q(o)$ when $q$ is the grounding of $p$: i.e., when
$\sigma(p)=q$. 
The way in which the   pool of grounded predicates $\P$  is constructed is similar to the
way in which a pool of unary predicates  is defined by  \citeay{bonet:aaai2019} for generating
Boolean and numerical features: there is a set of \emph{primitive predicates}
and a set of  description logic grammar rules \cite{description-logics} for defining new compound predicates from them.
The differences are that $\P$  contains  nullary and binary predicates as well,
and that the primitive predicates are not the domain predicates, that are to be learned,
but the   O2D predicates that are known and  grounded. The denotation of the predicates defined by a grammar rule
is determined by the semantics of the rule  and the denotation of predicates appearing  in the right  hand side.
The actual description logic (DL) grammar  considered for unary predicates
(concepts) is %given by the rules: %\footnote{\color{red}*** What about negation? ***}
\begin{alignat*}{1}
  C \ \leftarrow \ U \,\mid\, \top \,\mid\, \bot \,\mid\, \exists R.C \,\mid\, C \sqcap C' % \,\mid\, \neg C
\end{alignat*}
meaning that  unary O2D predicates ($U$), the universal true/false predicates,
existential restrictions, and intersections (conjunctions) %  negations
are all unary predicates. The rules for binary predicates (roles) are:
\begin{alignat*}{1}
  R \ \leftarrow \ R \,\mid\, R^{-1} \,\mid\, R \circ R'
\end{alignat*}
meaning that binary O2D predicates ($R_0$), inverses, and role  compositions are binary predicates.
Finally, nullary  predicates are obtained from unary predicates $C$ and $C'$ as $C \sqsubseteq C'$,
an expression that is true in states where the extension of $C$ is a subset of the extension of $C'$.
% These are all standard grammar rules in DLs logics where the denotation of the predicate
% on the left-hand side is determined by the denotation of the predicates on the right-hand
% side \cite{description-logics-refs}.

The  set of predicates of complexity no greater than $i$ is denoted as $\P_i$,  where
the complexity of top and bottom is $0$, the complexity of the  O2D predicates is $1$,
and the complexity of  derived predicates is $1$ plus the sum of the  complexities  of the
predicate  involved in the rule. For a given pool  of O2D states, the sequence $\P_0,\ldots,\P_m$ is constructed
iteratively, pruning 
% until reaching a \emph{fixpoint} while removing predicates with complexity bigger than $m$ and
duplicate  predicates (predicates with the same denotation). 
% $p$ in $\P_{i+1}$ is redundant if there is $q$
% in $P_j$, $j\leq i$, that has the same denotation over all the O2D states considered.
The pool $\P$ is  $\P_m$ for some bound $m > 0$. %% it doesnt make sense $m=0$

\smallskip\noindent\textbf{Example}.
In Blocks, the atom $\textit{clear}(b)$ is true when 
no block is above $b$ (not $\textit{some\_above}(b)$), and $b$ is not
held by the robot (not $holding(b)$). Given the O2D representation of Blocks in
Fig.~\ref{fig:o2d:examples}, $\textit{some\_above}$ and  $\textit{holding}$ can be
grounded to the  derived O2D predicates $\sigma(\textit{some\_above})=\exists\,below.block$ and $\sigma(\textit{holding})=\exists\,overlap.robot$,
both of complexity 2. 

\section{Learning: Formulation}

The training data $\D$ for learning grounded domains
is $\D=\tup{\T,\S,L,F}$, where $\T$ and $\S$ are sets of O2D states (scene representations)
over one or more instances, $\T\subseteq\S$;
$L$ is a set of action labels $\alpha$
(schema names), and $F_{\alpha}(s)$ is the {multiset} 
of O2D states $s'$ that follow $s$ in $\T$ when an action
with label $\alpha\in L$ is performed.
% The function $F_\alpha(s)$ encodes the successor states $s'$ that may follow a state $s$ in $\T$ after performing an action with name
% (label) $\alpha$, represented as a multiset.
Such states $s'$ are part of $\S$ but not necessarily of $\T$ that is a subset of $\S$.
In the formulation of \citeay{bonet:ecai2020}, the states in the data are black-boxes, not O2D states, and $\T=\S$.

% From these inputs, a \textbf{grounded domain} $D=\tup{\D,\sigma}$ is learned
% where $D$ contains one action schema (at most) for each label $\alpha in L$,
% and $\sigma$ maps each predicate symbol in $D$ into a grounded predicate
% $q=\sigma(p)$ in $\P$.

The grounded domain $\tup{D,\sigma}$ to be learned from this input contains
one action schema per action label $\alpha$, and determines a function
$h=h_\sigma^D$ that maps arbitrary O2D states $s$ into planning states $h(s)$ over $D$
with the same set of objects. More precisely, $h(s)$ is the truth valuation over the atoms $p(o)$,
where $p$ is a predicate in $D$ and $o$ is a \emph{tuple of objects}
from $s$ of the same arity as $p$, given by set of literals:
\begin{alignat*}{1}
  h(s) \ &\eqdef\ \{\, p(o) \,|\, q^s(o)=1, \sigma(p)=q, \text{$p$ in $D$}, \text{$o$ in $s$} \,\}\ \cup \\ % , arity(p)=arity(o)} \}
         &\quad\ \ \,\{\, \neg p(o) \,|\, q^s(o)=0, \sigma(p)=q, \text{$p$ in $D$}, \text{$o$ in $s$} \,\} % , arity(p)=arity(o)} \}
\end{alignat*}
which has positive literals $p(o)$ for $\sigma(p)\,{=}\,q$ and $q(o)$
true in $s$, and negative literals $\neg p(o)$ for $\sigma(p)\,{=}\,q$
and $q(o)$ false in $s$. 
% The function $h=h_\sigma^D$ maps indeed any image represented in O2D as the state $s$ into a
% planning state $h(s)$ over the domain $D$.
%
For action schema $\alpha$ in $D$, and planning state $\bar s$, let $F^D_\alpha(\bar s)$
represent the {multiset} formed by the states $\bar s'$ that follow $\bar s$ after
ground instantiations of the schema $\alpha$ in $\bar s$; i.e.,
\begin{alignat*}{1}
  F^D_\alpha(\bar s) \eqdef \multiset{\, \bar s' \,|\, \bar s'= f(a,\bar s), a \in A(\bar s), label(a)=\alpha}
\end{alignat*}
where $A(\bar s)$ is the set of ground instances of schema $\alpha$ over the objects in $\bar s$
that are applicable in $\bar s$, and $f$ is the state-transition function determined by $D$ for
the ground action $a$ in state $\bar s$.
The learning task can be formally defined as follows:

\begin{definition}[Learning Task]
  \label{def:task}
  Let $\D\,{=}\,\tup{\T,\S,L,F}$ be the input data, and let $\P$ be a pool of grounded predicates.
  The \textbf{learning task} $L(\D,\P)$ is to obtain a ({simplest}) grounded domain
  $\pair{D,\sigma}$ with one action schema per label $\alpha$ in $L$ such that the
  resulting abstraction function $h=h_\sigma^D$ complies with the following two constraints:
  %%Given input data $\D\,{=}\,\tup{\T,\S,L,F}$ and pool of grounded predicates $\P$,
  %%obtain a \emph{simplest} grounded domain $\pair{D,\sigma}$ with one action schema
  %%for label $\alpha$ in $L$ and groundings $\sigma(p)\,{\in}\,\P$ for all predicates
  %%$p$ in $D$, such that the resulting abstraction function $h=h_\sigma^D$ complies
  %%with the following constraints:
  \begin{enumerate}[C1.]
    \item If $s \not= s'$, then $h(s) \not= h(s')$, for $s, s' \in \T$; and
    \item $F^D_\alpha(h(s)) = \multiset{h(s') \,|\, s'\in F_\alpha(s)}$ for $s \in \T, \alpha \in L$.
    %\item $\multiset{\, h(s') \,|\, s' \in F_\alpha(s) \,} = \multiset{\, h(s') \,|\, h(s') \in F^D_\alpha(h(s)) \,}$, for all $s \in \T$, $\alpha \in A$; and
  \end{enumerate}
\end{definition}

The first constraint C1 says that the abstract (planning) states for different
O2D states in $\T$ must be different, while C2 says that the abstraction
function $h$ must represent an isomorphism.
Indeed, if $G_\D$ is the \emph{data graph} with vertex set $\S$ and edges
$(s,\alpha,s')$ for $s\,{\in}\,\T$, $s'\,{\in}\,F_\alpha(s)$ and $\alpha\,{\in}\,L$,
and $G_h$ is the \emph{planning graph} with vertex set $V_h$ equal to the planning
states reachable from $\{h(s)\,|\,\S\}$ and edges $(\bar{s}',\alpha,\bar{s})$ for
$\bar{s}\,{\in}\,V_h$, $a\,{\in}\,A(\bar{s})$, $label(a)=\alpha$, and $\bar{s}'\,{\in}\,F^D_\alpha(\bar{s})$,
then:\footnote{Proofs can be found in appendix.}
%Indeed, if $\T=\S$, and $G_{\D}$ represents the \emph{data graph} with vertex set
%$\S$ and edges $(s,\alpha,s')$ for $s'\in F_\alpha(s)$, and $G_h$, the \emph{planning graph} with vertices $h(s)$ for $s \in \S$
%and edges $(h(s),\alpha,h(s'))$ for $s'\in F_\alpha^D(s)$, $h=h_\sigma^D$, then:\footnote{Proofs can be found in appendix.}

\begin{theorem}
  \label{thm:isomorphism}
  If $\tup{D,\sigma}$ is a solution of the learning task $L(\D,\P)$ and $\T=\S$,
  the data and planning graphs $G_\D$ and $G_h$ for $h=h_\sigma^D$ are isomorphic.
  % with isomorphism $h$.
\end{theorem}

% Prefer models with minimum sum of actions' cardinalities
% #minimize { 1+N@10, A : a_arity(A,N) }.
% Prefer models with minimum sum of (non-static) predicates' cardinalities
% #minimize { 1+N@8, P : pred(P), p_arity(P,N), not p_static(P) }.
% Prefer models with minimum sum of (static) predicates' cardinalities
% #minimize { 1+N@6, P : pred(P), p_arity(P,N), p_static(P) }.
% Prefer models with minimum number of effects
% #minimize { 1@4, A, P, T, V : eff(A,(P,T),V) }.
% Prefer models with minimum number of preconditions
% #minimize { 1@2, A, P, T, V : prec(A,(P,T),V) }.

The \emph{complexity} of a domain $D$ is defined by a lexicographic
cost function that considers, in order, the arity of the action schemas, the
sum of the arities for non-static predicates, the same sum for static predicates,
the number of effects, and the number of preconditions.
The first three criteria are from \citeay{ivan:kr2021}.
% while the last two are novel and thus impose a more demanding optimization
% criterion.
The complexity of a grounded domain $\pair{D,\sigma}$ is the complexity of $D$,
and a grounded domain is \emph{simplest} when it has minimal complexity.
The \textbf{optimal solutions} of the learning task $L(\D,\P)$ in Definition~\ref{def:task}
are the simplest grounded domains that satisfy constraints C1 and C2.

Given a grounded domain $\tup{D,\sigma}$, any pair of O2D states $s_0$
and $s_g$ defines a classical planning problem $P\,{=}\,\pair{D,I}$ where
$I\,{=}\,\tup{O,Init,Goal}$ is such that the objects in $O$ are the ones in
$s_0$ and $s_g$, $Init\,{=}\,h(s_0)$, and $Goal\,{=}\,h(s_g)$.
%\footnote{A set of goal states can be defined in terms of multiple states $s_g$.}

\Omit{
  \textcolor{red}{*** Check: In Sokoban, there should be multiple goal states; with all possibilities of
  assigning box-ids to goal locations. For sokoban experiments, this is relevant. ALSO, we are assuming that
  every object has an id and can be tracked; ie if we swap two boxes, we get a different state. Eventually
  to be clarified. ***}
}

\subsection{Properties and Scope}

\Omit{
  The focus on the simplest domains $D$ that account for the dynamics expressed by
  the input data $\cal D$ makes the learned domains more likely to generalize to
  new, test instances. The \emph{verification} task over new instances is similar
  to the learning task but much simpler as the domain $D$ and the grounding $\sigma$
  are fixed to those that have been learned.
  A key question how complete is this learning framework is.
}

Some assumptions in the formulation are
1)~actions that change the planning state must change the O2D state (cf.\ C1),
2)~the objects in the planning instances are the ones appearing in the O2D
states, and 3)~the target language for learning is lifted STRIPS with negation.
These assumptions have concrete implications; e.g., in Sokoban, the cells in
the grid must appear as O2D objects, else assumption 1 is violated.
Likewise, in Sliding Tile, the tiles suffice for distinguishing O2D states,
but cells as objects are needed in STRIPS.\footnote{The problem of determining the
  ``objects'' in a scene for a {given target planning language} is subtle and not
  tied to our particular learning approach but to modeling in general.
  It also surfaces in deep learning approaches from images over {the same target languages}
  but then the problem does not become visible as the schemas and the objects are not transparent.
  A way out of this problem appears in \cite{bonet:ecai2020,ivan:kr2021} where the ``objects''
  are ``invented'' along with the action schemas and predicates.
}

The completeness of the approach can be characterized in terms of a ``hidden''
domain $D$. Namely, if the O2D states $s$ are mere ``visualizations'' of planning
states $\bar{s}$ over $D$, and there is a function $h=h_\sigma^D$ given the pool
of predicates $\P$ that allows us to recover the planning states $\bar{s}$ from
their visualizations, then the grounded domain $\tup{D,\sigma}$ is a solution
of the learning task $L(\D,\P)$:

\begin{theorem}
  \label{thm:learning}
  Let $D$ be a (hidden) planning domain, let $\D\,{=}\,\tup{\T,\S,L,F}$ be
  a dataset, and let $g$ be a 1-1 function that maps planning states $\bar{s}$
  in $D$ into O2D states $g(\bar{s})$ such that
  % 1)~$\S\subseteq\{g(s)\,|\,\text{$s$ is reachable in $D$}\}$
  % and 2)~$F_\alpha(g(\bar{s}))=\multiset{g(s')\,|\,\bar{s}'\in F^D_\alpha(\bar{s})}$ for $g(\bar{s})\in\T$ and $\alpha$ in $L$.
  $F_\alpha(g(\bar{s}))=\multiset{g(\bar{s}')\,|\,\bar{s}'\in F^D_\alpha(\bar{s})}$ for $g(\bar{s})\in\T$ and $\alpha$ in $L$.
  If there is a grounding function $\sigma$ for the predicates in $D$ over
  a pool $\P$ such that $h=h_\sigma^D$ is the right inverse of $g$ on $\T$
  (i.e., $g(h(s))=s$ for $s\in\T$),
  then $\tup{D,\sigma}$ is a solution for the learning task $L(\D,\P)$.
\end{theorem}

\Omit{ %Old statement
  \begin{theorem}
    \label{thm:learning}
    Let $D$ be a domain and $g(\bar{s})$ be a 1-to-1 onto function mapping
    planning states $\bar{s}$ over $D$ into a set $\T$ of O2D states $s=g(\bar{s})$ with the same set of objects,
    and let $F_\alpha(s) = \multiset{g(s') | s' \in F_\alpha^D(s)}$. If there is a grounding function $\sigma$ of the predicates in $D$ given the pool $\P$
    such that $h=h_\sigma^D$ is the inverse of $g$; i.e. $h(g(\bar{s}))=\bar{s}$, $g(\bar{s}) \in \T$,
    then $\tup{D,\sigma}$ is a solution of the learning task $L(\D,\P)$.
    %expressed by Def.~\ref{def:task}.
  \end{theorem}
}

\Omit{
  There is some relation between this learning formulation and those that learn
  action schemas \emph{given} the \emph{domain predicates}.
  In our formulation, it is the set of O2D predicates that is known,
  and the grounding of the domain predicates is assumed to lie among the 
  predicates $\P$ that can be derived from them.
}

The key difference from approaches that learn action
schemas \emph{given} the domain predicates is that, in our formulation, 
the domain predicates are not given but must be invented and grounded
using a pool of predicates that is obtained from the \emph{given}  O2D predicates.

\subsection{Extensions and Variations}

In some cases, we want an slight variation of the learning task $L(\D,\P)$
where there is no need to distinguish all O2D states (constraint C1).
For example, we may learn a relation $sep(s,s')$ that is true if
$s$ is a goal state and $s'$ is not, and then limit the scope of C1 to such pairs
(that need to be distinguished).
In other cases, the addition of \emph{domain constants} in the planning language
can reduce the arity of action schemas \cite{book:pddl}.
The constants are easily learned from O2D states where they correspond to the denotation
of grounded, unary, static predicates that single out one particular object per instance.
Such objects are identified at preprocessing and explicitly marked as constants before learning the action schemas.

\section{Learning: ASP Implementation}

The learning task $L(\D,\P)$  in Definition~\ref{def:task} can be cast as a combinatorial
optimization problem $T_{\beta}(\D,\P)$ once  two hyperparameters are set in  $\beta$: the  max  arity of  actions, 
and the  max number of predicates.
% .\textcolor{red}{*** AFAIK, the max pred arity 2 is fixed and determined by DL grammar and O2D; e.g. there is no way to form ternary preds ***}
The problem $T_{\beta}(\D,\P)$ is expressed and solved as an answer set
program \cite{brewka:asp,vladimir:asp,torsten:asp}  using the \clingo
solver \cite{torsten:clingo}, building on the code for learning \emph{ungrounded}
lifted STRIPS representations \cite{ivan:kr2021}. 
The main departures from \citeay{ivan:kr2021} are:
1)~there is no assumption that instances in the input data are represented
as full state graphs ($\T=\S$), 2)~there is no choice of the truth values of atoms $p(o)$
in the different nodes; instead a grounding $\sigma(p)$ for the domain predicates $p$ is  selected from $\P$
(actually, the name of the domain predicates is irrelevant and does not appear in the code); and
3)~action arguments of ground actions are factorized, so that if there are
actions of arity 4 and 15 objects, the $15^4=50,625$ ground actions are not enumerated.
 These changes allow us to learn domains that cannot be learned using
the previous methods.

Other departures are the use of O2D states in the input as opposed to black-box
states, the introduction of domain constants in the planning language, and a more
elaborated optimization criterion. The full ASP code is in the appendix.\footnote{Data and code will be made available.}

\section{Experimental Results}

We test the performance of the ASP program expressing the combinatorial optimization
problem $T_\beta(\D,\P)$ on  two  versions of Blocks and Towers of Hanoi, the Sliding-Tile Puzzle, IPC Grid, and Sokoban.
The pool of grounded predicates $\P$ is computed from the 
given  O2D predicates as mentioned above, using  complexity bounds  $m=2$ and  $m=4$ (details below).
The max number of predicates is set to $12$ and the maximum arity of actions is set to $3$
except for Sokoban that is set to $4$. The experiments are performed on Amazon EC2's
%\texttt{c5.9xlarge} instances that feature 48 Intel Xeon Platinum 8124M CPUs\,@\,3GHz, and 72GB of RAM,
\texttt{r5.8xlarge} instances that feature 32 Intel Xeon Platinum 8259CL CPUs\,@\,2.5GHz, and 256GB of RAM,
and \clingo is run with options `{\small\texttt{-t 6 --sat-prepro=2}}').
% The planner is Pyperplan \cite{pyperplan}.
%% Moved this below 

\Omit{
  The set of primitive predicates $\P_0$ used in the experiments consists of
  the five basic O2D binary predicates, together with the following unary
  predicates for object types: block, table, robot (from Blocksworld), disk,
  peg (from Towers of Hanoi), opencell, lockedcell (from IPC Grid, indicating whether
  a grid cell appears visually as open or closed; in Sokoban, all cells are open),
  crate and sokoban (from Sokoban). From $\P_0$, we define a pool of grounded
  predicates $\P=\bigcup_{i=0,2}\P_i$, which is used across all domains.
}

\smallskip
\noindent\textbf{Data generation.} The data $\cal D$ for learning and validation
is obtained from states $\bar{s}$ of planning instances $P_i=\tup{D,I_i}$, $i=1, \ldots, n$
for each domain, encoded in STRIPS and \emph{ordered} by the size of the state space.
The O2D states $s=g(\bar{s})$ are obtained from the planning states $\bar{s}$ using a 1-to-1 ``rendering'' function
$g$ as in Theorem~\ref{thm:learning} with $F_\alpha(g(\bar{s}))$ set to $\multiset{g(\bar{s}') \,|\, \bar{s}' \in f(a,\bar{s}), a \in A(\bar{s}), a \in \alpha(P_i)}$.
Characteristics of the data pool are shown in Table~\ref{tab:exp-input-data};
further details  can be found in the appendix (suppl. material). 
Sokoban1 and Sokoban2 refer to the same domain but different
training  instances: Sokoban2 contains fewer but much larger
instances (the largest  has 6,832 states).

\begin{table}[t]
  \centering
  \resizebox{\columnwidth}{!}{
    \begin{tabular}{@{\ }lrrrrrrrr@{\ }}
      \toprule
                       &         &           &       &        &         & \multicolumn{3}{c}{Predicate pool $\P$} \\
      \cmidrule{7-9}
      Domain (\#inst.) & \#obj.\ & \#const.\ & $|A|$ & $|\S|$ & \#edges & $|\P|$ & $m$ &   time \\
      \midrule
      %complexity2/blocks3ops 2 7 4 590 2414
      %complexity2/blocks4ops 3 8 5 1020 2414
      %feature_generation/domains/blocks3ops/complexity2/log.txt:2022-01-10 14:58:18,900 [INFO] [<module>:1034] [max-complexity=2] 7 role(s), 8 concept(s), and 13 predicate(s) in 0.071 second(s)
      %feature_generation/domains/blocks3ops/complexity2/log.txt:2022-01-10 14:58:18,932 [INFO] [<module>:1057] All tasks completed in 1.415 second(s)
      %feature_generation/domains/blocks4ops/complexity4/log.txt:2022-01-10 14:58:50,801 [INFO] [<module>:1034] [max-complexity=4] 36 role(s), 30 concept(s), and 79 predicate(s) in 4.348 second(s)
      %feature_generation/domains/blocks4ops/complexity4/log.txt:2022-01-10 14:58:51,381 [INFO] [<module>:1057] All tasks completed in 9.134 second(s)
      Blocks3ops (4)   &       5 &         2 &     3 &    590 &   2,414 &     13 &   2 &   1.41 \\
      Blocks4ops (5)   &       5 &         3 &     4 &  1,020 &   2,414 &     79 &   4 &   9.13 \\
      %complexity2/hanoi1op 1 9 5 363 1074
      %complexity2/hanoi4ops 1 9 5 363 1074
      %feature_generation/domains/hanoi1op/complexity2/log.txt:2022-01-10 14:58:19,516 [INFO] [<module>:1034] [max-complexity=2] 8 role(s), 6 concept(s), and 14 predicate(s) in 0.071 second(s)
      %feature_generation/domains/hanoi1op/complexity2/log.txt:2022-01-10 14:58:19,545 [INFO] [<module>:1057] All tasks completed in 2.028 second(s)
      %feature_generation/domains/hanoi4ops/complexity2/log.txt:2022-01-10 14:58:19,523 [INFO] [<module>:1034] [max-complexity=2] 8 role(s), 6 concept(s), and 14 predicate(s) in 0.072 second(s)
      %feature_generation/domains/hanoi4ops/complexity2/log.txt:2022-01-10 14:58:19,552 [INFO] [<module>:1057] All tasks completed in 2.035 second(s)
      Hanoi1op (5)     &       8 &         1 &     1 &    363 &   1,074 &     14 &   2 &   2.02 \\
      Hanoi4ops (5)    &       8 &         1 &     4 &    363 &   1,074 &     14 &   2 &   2.03 \\
      %complexity2/slidingtile 1 12 7 742 1716
      %feature_generation/domains/slidingtile/complexity2/log.txt:2022-01-10 14:58:18,431 [INFO] [<module>:1034] [max-complexity=2] 8 role(s), 10 concept(s), and 16 predicate(s) in 0.130 second(s)
      %feature_generation/domains/slidingtile/complexity2/log.txt:2022-01-10 14:58:18,482 [INFO] [<module>:1057] All tasks completed in 0.963 second(s)
      Sliding Tile (7) &      11 &         1 &     4 &    742 &   1,716 &     16 &   2 &   0.96 \\
      %complexity2/grid 1 9 17 801 1980
      %feature_generation/domains/grid/complexity4/log.txt:2022-01-10 14:58:49,063 [INFO] [<module>:1034] [max-complexity=4] 52 role(s), 73 concept(s), and 155 predicate(s) in 4.136 second(s)
      %feature_generation/domains/grid/complexity4/log.txt:2022-01-10 14:58:49,520 [INFO] [<module>:1057] All tasks completed in 7.274 second(s)
      %IPC Grid (17)    &       8 &         1 &    10 &    801 &   1,980 &    155 &   4 &   7.27 \\
      %
      % grid:new
      %feature_generation/domains/grid2/complexity4/log.txt:2022-01-12 11:10:23,042 [INFO] [<module>:1034] [max-complexity=4] 54 role(s), 80 concept(s), and 164 predicate(s) in 169.311 second(s)
      %feature_generation/domains/grid2/complexity4/log.txt:2022-01-12 11:10:30,413 [INFO] [<module>:1057] All tasks completed in 316.643 second(s)
      IPC Grid (19)    &      11 &         1 &    10 &  9,368 &  23,530 &    164 &   4 & 316.64 \\
      %complexity2/sokoban2 3 10 95 1936 5042
      %feature_generation/domains/sokoban2/complexity2/log.txt:2022-01-10 14:58:25,797 [INFO] [<module>:1034] [max-complexity=2] 8 role(s), 12 concept(s), and 18 predicate(s) in 0.733 second(s)
      %feature_generation/domains/sokoban2/complexity2/log.txt:2022-01-10 14:58:25,972 [INFO] [<module>:1057] All tasks completed in 8.454 second(s)
      Sokoban1 (95)    &      22 &         3 &     8 &  1,936 &   5,042 &     18 &   2 &   8.54 \\
      %complexity2/sokoban 3 30 24 12056 36482
      %feature_generation/domains/sokoban/complexity2/log.txt:2022-01-10 15:00:57,061 [INFO] [<module>:1034] [max-complexity=2] 8 role(s), 12 concept(s), and 18 predicate(s) in 6.562 second(s)
      %feature_generation/domains/sokoban/complexity2/log.txt:2022-01-10 15:00:57,997 [INFO] [<module>:1057] All tasks completed in 160.480 second(s)
      Sokoban2 (24)    &      27 &         3 &     8 & 12,056 &  36,482 &     18 &   2 & 160.48 \\
      \bottomrule
    \end{tabular}
  }
  \vskip -.4em
  \caption{Data pool. For each domain, columns show number of instances, max number of objects,
    number of domain constants, number of action labels, total number of states and edges across all instances,
    size of predicate pool $\P$, complexity bound $m$, and time in seconds to generate $\P$.
    Each instance consists of all states reachable from the initial state.
    %\textcolor{red}{(**Should we note here that the size of $\P$ is given in terms of \textit{non-redundant} preds? Otherwise they may wonder why the numbers are small**)}
    %Just added this in section Groundings
  }
  \label{tab:exp-input-data}
\end{table}

\smallskip
\noindent \textbf{Incremental learning.} The data generated from the planning instances
is used incrementally for computing an \emph{optimal solution} $\tup{D_i,\sigma_i}$ of
the learning task $L(\D_i,\P)$, $i=0, \ldots, n$. % for improved scalability.
$\D_0$ and $D_0$ are empty, and $\D_{i+1}$ is equal to $\D_i$, when the solution obtained
from $L(\D_i,\P)$ verifies (generalizes) over all the data in $\D$ (satisfies constraints
C1 and C2 in Definition~\ref{def:task}).
When not, $\D_{i+1}$ extends $\D_i$ with a set $\Delta$ of O2D states obtained from the
first $P_k$ instance where the verification fails.
If constraint C1 is violated for a pair of states $\{g(\bar{s}),g(\bar{s}')\}$, $\Delta$
is set to the pair. Else, if C2 is violated for some states $g(\bar{s})$ and label $\alpha$,
$\Delta$ collects up to the first 10 such states.
The set $\Delta$ extends $\D_{i}=\tup{\T_{i},\S_{i},L_{i},F^{i}}$ into $\D_{i+1}$ as
follows, where $\alpha \in D$ stands for the known action labels (schema names),
and $L_{i+1}$ is the set of all such labels for all $i > 0$:
\begin{enumerate}[--]
  \item $\T_{i+1} = \T_i \cup \Delta$,
  \item $\S_{i+1}=\S_i \cup \Delta \cup \bigcup\{ F_\alpha(g(\bar{s})) \,|\, g(\bar{s})\in\Delta, \alpha \in D \}$,
%   \item $A_{i+1}=A_i\cup\{ \alpha \,|\, \alpha\in P_j\}$, and
  \item $F^{i+1}_\alpha=F^i_\alpha \cup \{ \pair{g(\bar{s}),F_\alpha(g(\bar{s}))} \,|\, g(\bar{s})\in\Delta \}$, $\alpha \in D$.
\end{enumerate}

\Omit{
  The type of data required for learning varies from one domain to the other,
  based on how complex they are. As one would expect, in the case of ``simpler''
  domains such as Blocksworld, grounded domains that generalize can be obtained
  with data from a handful of instances.
  For more complex domains, with e.g. higher action arities, more action labels,
  or more complex preconditions and effects, a set of state transitions that is
  representative of domain dynamics is obtained from a larger number of instances.
  For example, in the experiments, a general grounded domain for Sokoban is learned
  from data coming from 97 instances. Table \ref{tab:exp-input-data} summarizes
  the input data used for the learning experiments.
}

% \smallskip

\smallskip
\noindent\textbf{Results.} Table \ref{tab:exp-statistics} shows the results of
the incremental learner given the pool of data in Table~\ref{tab:exp-input-data}.

For each domain, the columns show the number of iterations until
an optimal  model that verifies over all the instances in the data pool is found,
the number of instances and states (in $\T$) from the data pool used up to this point, 
%the number of states in the final $\T$ set,
and the times in seconds for grounding and solving the ASP program, for verification,
and total time. %for all together. % for the incremental solver.

\begin{table}[t]
  \centering
  \resizebox{\columnwidth}{!}{
    \begin{tabular}{@{\ }lrrr rrrr@{\ }}
      \toprule
                   &        &          &          & \multicolumn{4}{c}{Learning time in seconds} \\
      \cmidrule{5-8}
      Domain       & \#iter & \#inst.\ & \#states &    solve &   ground &verif.\ &    total \\ %. time (\% ground.) & Ver. time & Total time  \\
      \midrule
      %2022-01-10 13:16:09,544 [INFO] [solve:224] #calls= 5, solve_wall_time=    1.97, solve_ground_time=   1.92, verify_time=  0.84, elapsed_time=    2.97
      %2022-01-09 20:19:14,512 [INFO] [solve:224] #calls= 7, solve_wall_time=   23.66, solve_ground_time=  23.37, verify_time= 29.42, elapsed_time=   53.70
      Blocks3ops   &      5 &        3 &       20 &     0.05 &     1.92 &   0.84 &      2.97 \\
      Blocks4ops   &      7 &        3 &       16 &     0.29 &    23.37 &  29.42 &     53.70 \\
      %2022-01-10 13:16:08,729 [INFO] [solve:224] #calls= 4, solve_wall_time=    1.59, solve_ground_time=   1.53, verify_time=  0.44, elapsed_time=    2.16
      %2022-01-10 13:16:21,632 [INFO] [solve:224] #calls= 6, solve_wall_time=   14.23, solve_ground_time=  12.67, verify_time=  0.59, elapsed_time=   15.06
      Hanoi1op     &      4 &        2 &        7 &     0.06 &     1.53 &   0.44 &      2.16 \\
      Hanoi4ops    &      6 &        4 &       27 &     1.56 &    12.67 &   0.59 &     15.06 \\
      %2022-01-10 13:16:11,000 [INFO] [solve:224] #calls= 6, solve_wall_time=    3.00, solve_ground_time=   2.89, verify_time=  1.20, elapsed_time=    4.43
      Sliding Tile &      6 &        5 &       10 &     0.11 &     2.89 &   1.20 &      4.43 \\
      %2022-01-09 21:05:18,694 [INFO] [solve:224] #calls=25, solve_wall_time= 2795.50, solve_ground_time=2373.03, verify_time= 15.42, elapsed_time= 2817.88
     %IPC Grid     &     25 &       11 &       98 &   422.47 & 2,373.03 &  15.42 &  2,817.88 \\
      %2022-01-12 13:11:02,708 [INFO] [solve:227] #calls=27, solve_wall_time=4229.67, solve_ground_time=3536.23, verify_time=2404.87, elapsed_time= 6653.03, status=OK
      IPC Grid     &     27 &       12 &      127 &   693.44 & 3,536.23 & 2,404.87 &  6,653,03 \\
      %2022-01-10 13:21:39,964 [INFO] [solve:224] #calls=10, solve_wall_time=  301.74, solve_ground_time= 285.56, verify_time=  9.18, elapsed_time=  311.79
      %2022-01-10 16:48:48,603 [INFO] [solve:224] #calls=11, solve_wall_time=12565.02, solve_ground_time=5314.35, verify_time=165.19, elapsed_time=12740.43
      Sokoban1     &     10 &        9 &       13 &    16.18 &   285.56 &   9.18 &    311.79 \\
      Sokoban2     &     11 &        8 &       56 & 7,250.67 & 5,314.35 & 165.19 & 12,740.43 \\
      \bottomrule
    \end{tabular}
  }
  \vskip -.4em
  \caption{Learning results. For each domain, the first column shows the number of
    iterations of the incremental learner until optimal solutions %of the learning task $L(\D_i,\P)$
    that verify over all data in the pool are found.
    The others show
    the number of instances and states in the final set $\T$ constructed from the data pool,
    and the times in seconds for solving and grounding the ASP programs, for
    verification, and  total time.
  }
  \label{tab:exp-statistics}
\end{table}

The learning task $L(\D,\P)$ for all domains admit solution with the pool $\P\,{=}\,\P_m$
for $m\,{=}\,2$, except for IPC Grid and Blocks4ops where no solution exists for
$m\,{\leq}\,3$ and require a bound $m\,{=}\,4$.
In both cases, however, the solver takes less than 20 seconds in total to report
lack of solutions for the bounds $m\,{=}\,2$ and $m\,{=}\,3$.

Some of the domains have been considered before, like Blocks3ops and Hanoi1op
\cite{bonet:ecai2020,ivan:kr2021}, but others, like IPC Grid and
Sokoban are more challenging.
% On the other hand, the new domains IPC Grid and Sokoban are likely to be out of
% reach of previous approaches given the high number of objects, the high number
% of actions, and the high arity of actions.
The final model for IPC Grid involves 10 action schemas (6 of arity 2 and 4 of arity 3),
while the one  for Sokoban involves   8 action schemas (4 of arity 2 and 4 of arity 4).
The max number of objects  that ended up being used  during training was
8 for IPC Grid and 21 for Sokoban2.
% (the dataset for the latter includes instances
% with up to 27 objects, but such instances did not make it into the training set
% and end up being used only for verification).
%%, and the training data required
%involved up 7 objects in the first, and up to 17 objects in the second
%(the data pool included instances with up to 20 objects but their states
%didn't make it into the training data of the incremental leaner).

\subsection{Learned Representations}

In the experiments, data obtained from  hidden planning instances
was  used for  generating the training data. % and learning the  grounded domains. 
The original and learned domains, referred to as   $D_O$ and $D_L$,
must agree on the number and name of the action schemas, but not in their arities
or in the predicates involved. % , and the learned domain $D_L$ can use constants.
Table~\ref{table:geometry} compares $D_O$ and $D_L$ along 
dimensions reflected in  the optimization criterion, and Fig.~\ref{fig:schema:example}
shows learned schemas for IPC Grid and Sokoban. 

In general, the learned domains are not equal to the hidden domains,
but they are close and equally meaningful and interpretable.

The groundings obtained for the predicates of the different domains are
interesting as well (predicates names are our own). For example, Sokoban uses `$nempty(c)$' atoms that hold  when cell $c$ has
either a crate or the sokoban, and `$at(x,y)$' atoms  that hold when object
$x$ is at  $y$; the first is grounded on the derived O2D predicate
`$\exists\,overlap.\top$' of complexity 2, and the second %predicate
as `$overlap$' of complexity~1. More complex groundings are obtained in IPC Grid.
For example, the following  groundings have all  complexity 4:
the nullary `$armempty$' predicate that holds when the robot holds no key,
is grounded on the derived O2D predicate  `$key\,{\sqsubseteq}\,\exists\,overlap.\top$' (i.e., all keys are in cells); 
the unary predicate `$somecell(\cdot)$' that holds for a key $k$ if $k$ is in some cell, 
is grounded on  `$key\,{\sqcap}\,\exists\,overlap.\top$', 
and the binary predicate  `$match(\cdot,\cdot)$' that holds when key $k$
has the shape of the lock at cell $c$,   is grounded on  `$shape\,{\circ}\,shape^{-1}$'
(a binary relation that holds for two objects of the same shape).
%\textcolor{red}{*** Is it true that in Grid, there cannot be two cells with the same shape, or two keys? ***}

\begin{table}[t]
  \centering
  \resizebox{\columnwidth}{!}{
    \begin{tabular}{@{\ }lrr@{\ }rrrr@{\ }}
      \toprule
                   & \multicolumn{2}{c}{Original domain $D_O$} && \multicolumn{3}{c}{Learned domain $D_L$} \\
      \cmidrule{2-3} \cmidrule{5-7}
      Domain       &   action arities & \#pred.\ &&   action arities & \#pred.\ & \#$c$ \\
      \midrule
      Blocks3ops   &      (2 of 2, 3) &  $(3,1)$ &&      (2 of 2, 3) &  $(2,0)$ &     1 \\[.1em]
      Blocks4ops   & (2 of 1, 2 of 2) &  $(5,0)$ && (2 of 1, 2 of 2) &  $(3,0)$ &     2 \\[.1em]
      Hanoi1op     &         (1 of 3) &  $(2,3)$ &&         (1 of 3) &  $(2,1)$ &     0 \\[.1em]
      Hanoi4ops    &         (4 of 3) &  $(2,3)$ &&         (4 of 3) &  $(2,2)$ &     0 \\[.1em]
      Sliding Tile &         (4 of 3) &  $(4,2)$ &&         (4 of 3) &  $(2,2)$ &     0 \\[.1em]
      IPC Grid     & (6 of 2, 4 of 4) &  $(6,7)$ && (6 of 2, 4 of 3) &  $(4,4)$ &     1 \\[.1em]
      Sokoban      & (4 of 3, 4 of 5) &  $(2,4)$ && (4 of 2, 4 of 4) &  $(2,2)$ &     1 \\
      \bottomrule
    \end{tabular}
  }
  \vskip -.4em
  \caption{Comparison of original, hidden domains used to generate the data  ($D_O$) and learned domains ($D_L$).
    The columns shown action arities, number of dynamic and static predicates (\#pred), and
    number of constants in $D_L$ (\#c). E.g., $D_O$ for Blocks4ops has 2 actions of arity 1 (Pickup
    and Putdown), 2 actions of arity 2 (Stack and Unstack), 5 dynamic predicates  (ontable, on, holding, clear, and armempty), and no static predicates.
    The learned grounded domains for both Sokoban benchmarks are equal; only one is shown. 
  }
  \label{table:geometry}
\end{table}

\begin{figure}[t]
  \centering
  \resizebox{\columnwidth}{!}{
    \fbox{
      \begin{minipage}{1.10\columnwidth}\small\tt
        \noindent[Grid]\,\textcolor{acolor}{Pickup$(p,k)$:} \\[.1em]
        \textcolor{pcolor}{pre:} $armempty$, $at(\text{R},p)$, $at(p,k)$ \\[.1em]
        \textcolor{pcolor}{eff:} $\neg armempty$, $\neg somecell(k)$, $\neg at(p,k)$, $\neg at(k,p)$ \\
        %\hrule\medskip
        %\noindent\textcolor{red!80!black}{Pushdown$(1,2,3,4)$:} \\[.1em]
        %\textcolor{green!60!black}{static:} $below(1,3)$, $below(4,1)$ \\[.1em]
        %\textcolor{green!60!black}{pre:} $at(2,1),\, at(\text{Sok},3),\, \neg nempty(4)$ \\[.1em]
        %\textcolor{green!60!black}{eff:} $\neg nempty(3),\, nempty(4),\, at(\text{Sok},1),\, at(1,\text{Sok}),\, \neg at(\text{Sok},3)$, \\[.1em]
        %\textcolor{white}{eff:}          $\neg at(3,\text{Sok}),\, \neg at(1,2),\, \neg at(2,1),\, at(2,4),\, at(4,2)$
        \hrule\medskip
        \noindent[Sokoban]\,\textcolor{acolor}{Pushdown$(x,y,z,c)$:} \\[.1em]
        \textcolor{pcolor}{static:} $below(z,y)$, $below(y,x)$ \\[.1em]
        \textcolor{pcolor}{pre:} $at(\text{Sok},x),\, at(c,y),\, \neg nempty(z)$ \\[.1em]
        \textcolor{pcolor}{eff:} $\neg nempty(x),\, nempty(z),\, at(\text{Sok},y),\, at(y,\text{Sok}),\, \neg at(\text{Sok},x)$ \\[.1em]
        \textcolor{white}{eff:}          $\neg at(x,\text{Sok}),\, \neg at(y,c),\, \neg at(c,y),\, at(c,z),\, at(z,c)$
      \end{minipage}
    }
  }
  \vskip -.4em
  \caption{Two learned action schemas for Grid (top) and Sokoban (bottom).
    Predicates names our own; see text for their grounding.
%     Commented this:
%     Pickup$(p,k)$ picks up key $k$ at position $p$: the robot R and key $k$
%     must be at position $p$, and R should be carrying no key.
%     The effect is for the key to be held by robot and not at $p$.
%     Pushdown$(x,y,z,c)$ moves the crate $c$ that is at cell $y$ into cell $z$
%     with the Sokoban that is at cell $x$. The cell $z$ must be right below $y$,
%     and this right below $x$, the Sokoban must be at $x$, the crate at $y$ and $z$ must be empty.
%     The action results in the crate being at $z$, the Sokoban at $y$, and $x$ empty.
%    (See the text for the grounding of the predicates in the schemas.)
  }
  \label{fig:schema:example}
\end{figure}

\subsection{\label{subsec:planning}Planning with Learned Grounded Domains}

The computational value of learning grounded domains $\tup{D,\sigma}$ can be illustrated
by using them to solve new instances $P\,{=}\,\tup{D,I}$ expressed in terms of pairs of
O2D states, $s_0$ and $s_g$, for the initial and goal situations encoded as $h(s_0)$ and
$h(s_g)$ for $h=h_D^\sigma$. The new instances may  involve  sets of objects $O$ that are much larger
than those used in  training. %The set of objects $O$ in $s_0$ and $s_g$ can be much larger than the sets used in training.
% The instance information $I$ contains the objects $O$, and the
% initial and goal planning states are set to 
The resulting instances are standard and can be solved with any off-the-shelf planner.
%as illustrated below.

A plan $\pi=\tup{a_0,\ldots,a_n}$ that solves such an instance $P$ can be used to compute
the corresponding sequence of O2D states $s_0, \ldots, s_{n+1}$, with $s_G=s_{n+1}$, as follows.
If $\bar{s}_0,\ldots, \bar{s}_{n+1}$ are the planning states visited by $\pi$
with $\bar{s}_0=h(s_0)$, and the label of $a_0$ is $\alpha$, $s_1$ is the possible $\alpha$-successor
of $s_0$ such that $h(s_1)=\bar{s}_1$. The successors $s_2$,\ldots, $s_{n+1}$ are
selected in the same way.\footnote{The method assumes
  a simulator that given an O2D state $s$ produces the possible $\alpha$-successors of $s$. %; i.e., $F_{\alpha}(s)$.
  In the experiments, the simulator is determined by a ``hidden'' domain
  and the function $g(\cdot)$ that maps planning states into O2D states, but a different one could be used potentially
  where the O2D states are obtained from images. 
}
This method of ``applying'' the plans obtained from the learned grounded domain
provides an extra verification: if there is no $\alpha$-successor $s_{i+1}$ with $h(s_{i+1})=\bar{s}_{i+1}$
or $s_{n+1} \not= s_g$, the learned domain or its grounding is not  generalizing  to the new instance.
The fact that this does not happen in the experiments below is thus  additional evidence that the learned grounded domains
are correct.\footnote{Equivalence can also be  proved formally.} %  but this is beyond the scope of this work. \textcolor{red}{\bf REMOVE?}}

\Omit{
  %%% Sequences of O2D states are not ``visualizations'' exactly .. 
Following this procedure, a `visualization' of the states reached via $\pi$ is obtained.
We report tests with this method over a large instance of Sokoban, and 5 instances
of Blocks4ops of increasing size, as it provides an indirect way to verify the correctness
of the learned grounded domains.

%To illustrate this method, we took instances of Blocksworld (4ops) and Sokoban that were
%larger and more complex than the  ones used for training, and computed their corresponding
%O2D state sequences.
%This method provides an alternative way to verify the learned grounded domains.
Indeed, if there is no $\alpha$-successor $s_{i+1}$ with $h(s_{i+1})=\bar{s}_{i+1}$
or $s_{n+1} \not= s_g$, then the learned domain or its grounding would not generalize
to the new instance.
The fact that this does not happen is additional evidence that the learned grounded domains
are correct.\footnote{Equivalence between the hidden and the learned domains can also be
  proved formally, but this is beyond the scope of this work.
}
}

\begin{figure}[t]
  \centering
  \begin{tabular}{ccc}
    \includegraphics[width=0.385\columnwidth]{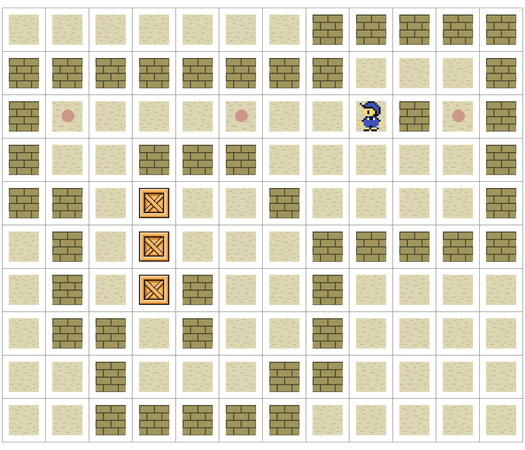}
    &\quad&
    \includegraphics[width=0.385\columnwidth]{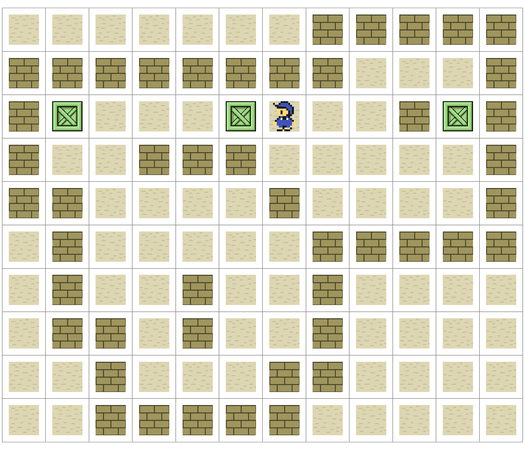} \\[-.4em]
    \small Initial state && \small Goal state
  \end{tabular}
  \vskip -.4em
  \caption{Depiction of initial and goal O2D  states for a  large Sokoban instance. % solved using the learned representation.
    Optimal plans of length 156 are found using the original
    ``hidden'' domain  $D_O$ and the learned grounded domain $D_L$.
  }
  \label{fig:o2d:planning}
\end{figure}

Figure~\ref{fig:o2d:planning} depicts the initial and goal O2D states $s_0$
and $s_g$ of a  large Sokoban instance. A plan $\pi$ of length 288
was obtained  from the planning instance $P=\tup{D,I}$,  where $D$
is the  learned domain, and $h(s_0)$ and $h(s_g)$ replace $s_0$ and $s_g$.
The plan was found with   Pyperplan \cite{pyperplan} running  a greedy best-first search
guided with the additive heuristic.  For each state $\bar{s}_i$ generated by the plan $\pi$,
a matching O2D state $s_i$ was found as above,  and  $s_n=s_g$.
The same verification was carried out in Blocks4ops instances with 7, 10, 15, 20 and 25 blocks,
some producing plans with up to 121 actions.
% are considered. The same planner generates successful plans with matching
% visualizations in all cases, including one with 121 actions.
%A sequence of matching $\alpha$-successor states was successfully obtained for all such plans.

We also compared  the performance of planners on
instances $P\,{=}\,\tup{D_O,I}$, where $D_O$ is  the hidden domain 
used to generate the data and  $I\,{=}\,\tup{O,\bar{s}_0,\{\bar{s}_g\}}$,
and the  corresponding instances $P'\,{=}\,\tup{D_L,I'}$, where $D_L$ is the learned domain
and $I'$ replaces the initial and goal states $\bar{s}$ by $h(g(\bar{s}))$,
a mapping that uses the ``rendering'' function used to generate the O2D states
(see appendix) and the learned function $h=h_\sigma^D$. If the learned domains are
correct, the state graphs associated to $P$ and $P'$ should be isomorphic 
and the optimal plans should have the same length (but the plans
themselves do not have to be the same).  We tested this
in three large   instances of Sokoban and of  Blocks4ops
using an optimal planner that runs A* with the LM-cut heuristic
\cite{helmert2009landmarks}. For the Sokoban instance shown in Fig.~\ref{fig:o2d:planning},
an optimal plan of length 156 was  found in 27 seconds for $P$
and in  54 seconds for $P'$. For two other instances, optimal
plans of length 134 and 135 were  found in 65 and 849 seconds for $P$,
and in 130 and 1,520 seconds for $P'$. Similar results were obtained
for the Blocks4ops instances.

\section{Summary}

We have introduced a formulation for learning crisp and meaningful first-order planning domains from
parsed  visual representations  that are not far from those   produced by object detection modules.
For this, the formulation for learning domains (action schemas and predicates) from the structure of
the state space  \cite{bonet:ecai2020,ivan:kr2021} was taken 
to a new setting where the traces do not have to be complete and the states observed
are not black boxes but parsed images in O2D. Two results are that the  learned planning representations
are grounded in O2D states, and hence  new problems can be given in terms of pairs of
O2D states representing the initial and goal situations, and that the learning scheme
scales up better than previous ones, enabling us to learn  more challenging domains
like the Sliding-tile puzzle, IPC Grid, and Sokoban. We have also run planning experiments
using the learned domains and their grounding functions that illustrate that the learned domains
can be used with off-the-shelf planners and are not too different than the domains that are written
and grounded by hand. 

\section{Acknowledgements}

This work was partially supported by ERC Advanced Grant No. 885107, by project TAILOR,
Grant No. 952215, funded by EU Horizon 2020, and by  the Wallenberg AI, Autonomous Systems
and Software Program (WASP) program, funded by the Knut and Alice Wallenberg Foundation, Sweden.

% \newpage 
\small
\bibliographystyle{named}
\bibliography{control,bib}
\normalsize

\appendix

\section{Appendix}

This appendix contains the proofs of the theorems, further details about the data generation,
a full description of the grounded domains learned, the full code of the learner (ASP program $T_\beta(\D,\P)$),
and additional details of the verifier and of the function $g(\cdot)$ used to generate the data from
hidden planning domains. 

\subsection{Proofs}

\setcounter{definition}{1}

Let us recall the definition and theorem statements:

\begin{definition}[Learning Task]
  \label{def:task}
  Let $\D\,{=}\,\tup{\T,\S,A,F}$ be the input data, and let $\P$ be a pool of grounded predicates.
  The \textbf{learning task} $L(\D,\P)$ is to obtain a ({simplest}) grounded domain
  $\pair{D,\sigma}$ with one action schema per label $\alpha$ in $A$ such that the
  resulting abstraction function $h=h_\sigma^D$ complies with the following two constraints:
  %%Given input data $\D\,{=}\,\tup{\T,\S,A,F}$ and pool of grounded predicates $\P$,
  %%obtain a \emph{simplest} grounded domain $\pair{D,\sigma}$ with one action schema
  %%for label $\alpha$ in $A$ and groundings $\sigma(p)\,{\in}\,\P$ for all predicates
  %%$p$ in $D$, such that the resulting abstraction function $h=h_\sigma^D$ complies
  %%with the following constraints:
  \begin{enumerate}[C1.]
    \item If $s \not= s'$, then $h(s) \not= h(s')$, for $s, s' \in \T$; and
    \item $F^D_\alpha(h(s)) = \multiset{h(s') \,|\, s'\in F_\alpha(s)}$ for $s \in \T, \alpha \in A$.
    %\item $\multiset{\, h(s') \,|\, s' \in F_\alpha(s) \,} = \multiset{\, h(s') \,|\, h(s') \in F^D_\alpha(h(s)) \,}$, for all $s \in \T$, $\alpha \in A$; and
  \end{enumerate}
\end{definition}

For a given dataset $\D=\tup{\T,\S,L,F}$, the data
graph $G_\D$ has vertex set $V_\D=\S$ and labeled edges
$E_\D = \{(s,\alpha,s') \,|\, s\in\S, s'\in F_\alpha(s), \alpha\in L\}$.
On the other hand, for planning domain $D$ and function $h$ that maps
states $s$ in $\S$ into planning states $\bar{s}$ in $D$, the planning
graph has as vertex set $V_h$ the set of reachable planning states from
$\{h(s)\,|\, s\in \S\}$, and as labeled edges $E_h$ the set
$\{(\bar{s},\alpha,\bar{s}')\,|\, \bar{s}\in V_h, a\in A(\bar{s}), label(a)=\alpha, \bar{s}'\in F^D_\alpha(\bar{s}) \}$.

\begin{theorem}
  If $\tup{D,\sigma}$ is a solution of the learning task $L(\D,\P)$ and $\T=\S$,
  the data and planning graphs $G_\D$ and $G_h$ for $h=h_\sigma^D$ are isomorphic.
  % with isomorphism $h$.
\end{theorem}
\begin{proof}%[Proof of Theorem~\ref{thm:isomorphism}]
  We first show that the function $h$ is a bijection from $V_\D$ onto $V_h$,
  and then show that the multisets of labeled edges are preserved by $h$.

  By construction of $D$, the set of labels in both graphs are equal, and by
  constraint C1 in Def.~\ref{def:task}, the function $h:V_\D\rightarrow V_h$ is 1-1.
  To show that $h$ is onto, we show $|V_h|\,{\leq}\,|V_\D|$.
  For a proof by contradiction, suppose $|V_\D|\,{<}\,|V_h|$.
  Then, either $V_h$ contains a vertex not reachable from $\{h(s)\,|\,s\in\S\}$,
  or there is a vertex $\bar{s}$ in $V_h$ and label $\alpha$ in $L$ such that
  \begin{alignat*}{1}
    |\multiset{ h(s') \,|\, s'\in F_\alpha(s) }| < |F^D_\alpha(h(s))| \,.
  \end{alignat*}
  The first case is impossible by definition of $G_h$.
  In the second case, since $s'\in F_\alpha(s)$ implies $h(s')\in F^D_\alpha(h(s))$
  (by constraint C2), then there is a state $s'\in\S$ such that $s'\notin F_\alpha(s)$
  and $h(s')\in F^D_\alpha(h(s))$, which also contradicts C2.

  Finally, to show that $h$ preserves edges, let $s$ and $s'$ be two states in $\S$,
  and let $\alpha$ be an action label.
  If $(s,\alpha,s')\in G_\D$, then $h(s')\in F^D_\alpha(h(s))$ by $C2$.
  Likewise, if $(h(s),\alpha,h(s'))\in G_h$, then $s'\in F_\alpha(s)$ also by $C2$.
  Hence, $h$ preserves edges and $G_\D$ and $G_h$ are isomorphic.
\end{proof}

\begin{theorem}
  Let $D$ be a (hidden) planning domain, let $\D\,{=}\,\tup{\T,\S,L,F}$ be
  a dataset, and let $g$ be a 1-1 function that maps planning states $\bar{s}$
  in $D$ into O2D states $g(\bar{s})$ such that
  % 1)~$\S\subseteq\{g(s)\,|\,\text{$s$ is reachable in $D$}\}$
  % and 2)~$F_\alpha(g(\bar{s}))=\multiset{g(s')\,|\,\bar{s}'\in F^D_\alpha(\bar{s})}$ for $g(\bar{s})\in\T$ and $\alpha$ in $L$.
  $F_\alpha(g(\bar{s}))=\multiset{g(\bar{s}')\,|\,\bar{s}'\in F^D_\alpha(\bar{s})}$ for $g(\bar{s})\in\T$ and $\alpha$ in $L$.
  If there is a grounding function $\sigma$ for the predicates in $D$ over
  a pool $\P$ such that $h=h_\sigma^D$ is the right inverse of $g$ on $\T$
  (i.e., $g(h(s))=s$ for $s\in\T$),
  then $\tup{D,\sigma}$ is a solution for the learning task $L(\D,\P)$.
\end{theorem}
\begin{proof}
  We need to show that the grounded domain $\tup{D,\sigma}$ complies with
  the constraints C1 and C2 in Definition~\ref{def:task}.

  For C1, let $s$ and $s'$ be different states in $\T$.
  If $h(s)\,{=}\,h(s')$, then $s=g(h(s))=g(h(s'))=s'$ by the condition on $g$.
%   Hence, $h(s)\neq h(s')$. Not needed 

  Let $s\in\T$ be an O2D state, and let $\alpha\in L$ be an action label.
  First notice that for planning state $\bar{s}$, $g(h(g(\bar{s})))=g(\bar{s})$ implies $h(g(\bar{s}))=\bar{s}$ and thus $h$ is a left inverse of $g$.
  Then,
  \begin{alignat*}{2}
    F^D_\alpha(h(s))\ &=\ \multiset{ \bar{s}' \,|\, \bar{s}' \in F^D_\alpha(h(s)) }         & \ \ \ & \text{\small(definition)} \\
                      &=\ \multiset{ \bar{s}' \,|\, g(\bar{s}') \in F_\alpha(g(h(s))) }     & \ \ \ & \text{\small(def.\ $F_\alpha$ in Thm)} \\
                      &=\ \multiset{ \bar{s}' \,|\, g(\bar{s}') \in F_\alpha(s) }           & \ \ \ & \text{\small(right inv.)} \\
                      &=\ \multiset{ h(g(\bar{s}')) \,|\, g(\bar{s}') \in F_\alpha(s) }     & \ \ \ & \text{\small(left inv.)} \\
                      &=\ \multiset{ h(s') \,|\, s' \in F_\alpha(s) } \,.                   & \ \ \ & \text{\small(def.\ $F_\alpha$ in Thm)}
  \end{alignat*}
  Therefore, constraint C2 is satisfied as well.
\end{proof}

\subsection{Data Generation: Details}

\paragraph{Blocks3ops and Blocks4ops.}
Two encodings of the classical planning domain, where stackable blocks
need to be reassembled on a table by a robot. Instances are parametrized
by the number $n$ of blocks.
The instances in the dataset have $n=1,\dots,5$ blocks.
Blocks3ops has 3 action labels (Stack, Newtower, and Move) while
Blocks4ops has 4 (Pickup, Putdown, Unstack, and Stack).
O2D states are defined based on the corresponding PDDLGym state images
for this domain, as illustrated in Figure~2. %\ref{fig:o2d:examples}.

\paragraph{Hanoi1op and Hanoi4ops.}
Two encodinfs of the Tower of Hanoi problem with arbitraty number of pegs and disks.
The datasets in both cases contain the instances for 3 pegs and $n$ disks, $n=1,\ldots,5$.
%This well-known involves three pegs and $n$ disks.
%The objective is to get the disks from the leftmost peg to rightmost peg by moving the disks,
%one at a time, from one peg to another. Only top disk on any stack can
%be moved, and a bigger disk cannot be placed on a smaller disk.
%Instances are parametrized by the number $n$ of disks.
%The instances used for learning have $n=1,\dots,5$ disks.
Hanoi1op involve a single action label Move while Hanoi4ops has 4 labels:
MoveFromPegToPeg, MoveFromPegToDisk, MoveFromDiskToPeg, and MoveFromDiskToDisk.

\paragraph{Sliding Tile.}
The generalization of the 15-puzzle problem over rectangular grids
of arbitrary dimensions, parametrized as $r\,{\times}\,c$ where $r$ and $c$
are the number of rows and columns.
%This is a combination puzzle in which tiles are arranged in a rectangular
%grid. Tiles can be slided along certain routes to change their position.
%The goal is to slide the tiles around the grid until a given end-configuration
%is achieved. Instances are parametrized by the number of rows $r$ and columns
%$c$ of the puzzle grid.
The dataset contains instances $r\,{\times}\,c$ such that the number of cells $rc\leq 6$.
The action labels are MoveUp, MoveRight, MoveDown, and MoveLeft.  O2D states are defined based on images such as the one illustrated in Figure~\ref{fig:slidingtile_appendix}.

\begin{SaveVerbatim}[numbers,gobble=0,commandchars=\\\{\}]{o2d:slidingtile}
  \textcolor{bcolor}{% object and types}
  tile(t1). tile(t2). 
  cell(c1_1). cell(c1_2).
  ...
  \textcolor{bcolor}{% relations}
  overlap(t1,c1_1).
  below(c2_1,c_1_1).
  ...
  \textcolor{bcolor}{% shapes}
  shape(d1,rectangle).
  ...
\end{SaveVerbatim}

\begin{figure}[t]
  \centering
  \begin{tabular}{c@{\quad}c}
   \includegraphics[width=0.14\textwidth]{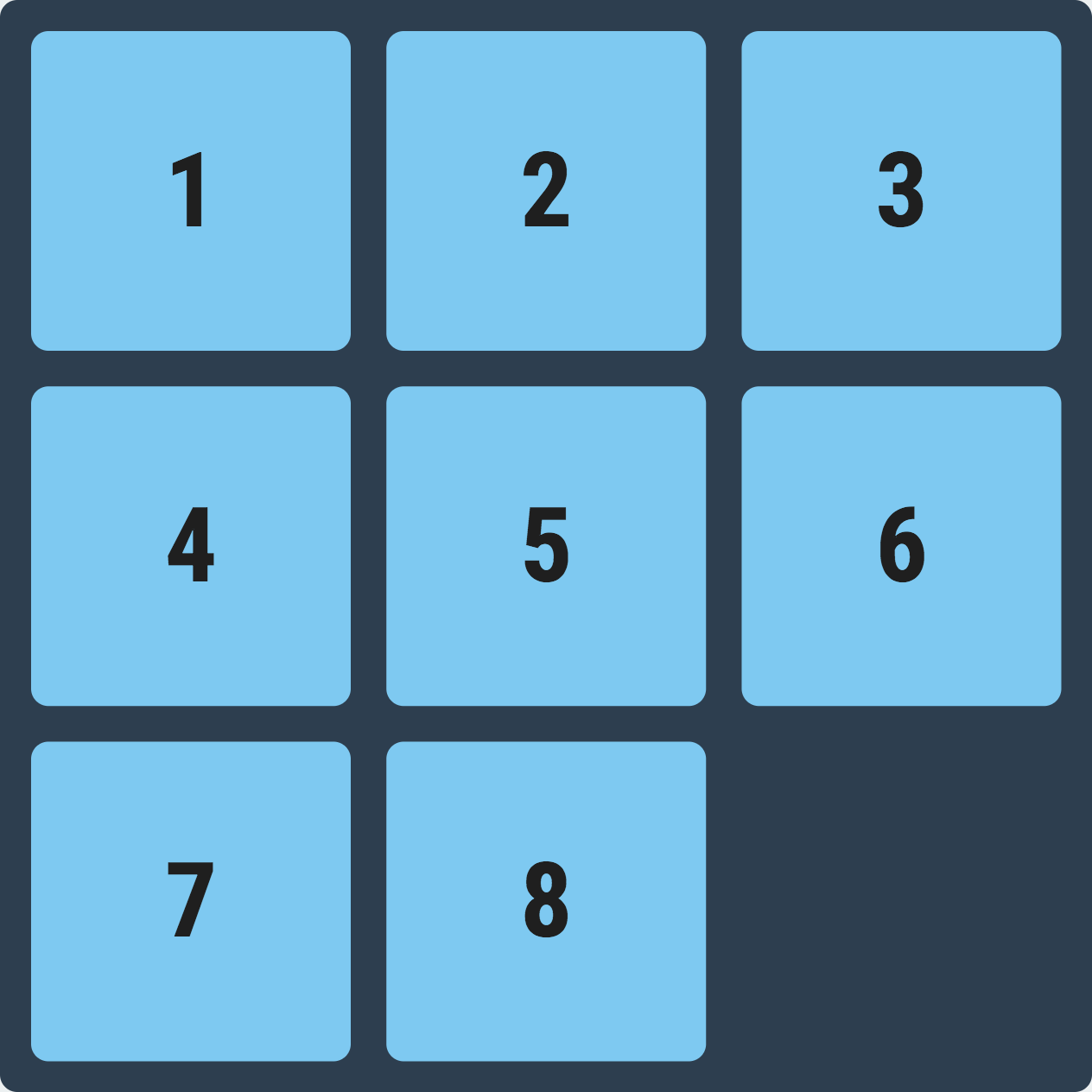} &
    \resizebox{0.22\textwidth}{!}{\BUseVerbatim{o2d:slidingtile}} 
  \end{tabular}
  \vskip -.5em
  \caption{Slidingtile scene and corresponding O2D state.}
  \label{fig:slidingtile_appendix}
\end{figure}

\paragraph{IPC Grid}
In this planning problem from the Int.\ Planning Competition (IPC), there is a
robot that moves within a rectangular grid where cells may be locked, but that
can be opened with matching keys, where a key and a cell match if they have the
same shape.
Keys can be picked and dropped by the robot, and locked cells can be opened
with the right key from an adjacent cell. The goal is to have some of the keys
at specified locations.
Instances are parametrized by the number of rows $r$ and columns $c$ of the grid,
the number of key/cell shapes $s$, the number of keys $k$, and the number of locked
cells $\ell$.
The instances used for learning are generated with $r\leq2$, $c\leq2$, $s\leq2$, $k\leq2$ and $\ell\leq1$.
For each combination of parameters, one instance is generated, in which the locations
of objects is randomized.
The action space has 10 labels: MoveUp, MoveRight, MoveDown, MoveLeft, Pickup,
Putdown, UnlockFromAbove, UnlockFromRight, UnlockFromBelow, and UnlockFromLeft. O2D states are defined based on images such as the one illustrated in Figure~\ref{fig:grid_appendix}.

%grid-1rows-2cols-0shapes-0keys-0locks.lp      
%grid-1rows-2cols-1shapes-1keys-0locks.lp      
%grid-1rows-2cols-1shapes-2keys-0locks.lp      
%grid-1rows-2cols-2shapes-2keys-0locks.lp      
%grid-1rows-5cols-2shapes-4keys-3locks.lp
%grid-2rows-1cols-0shapes-0keys-0locks.lp
%grid-2rows-1cols-1shapes-1keys-0locks.lp
%grid-2rows-1cols-1shapes-2keys-0locks.lp
%grid-2rows-1cols-2shapes-2keys-0locks.lp
%grid-2rows-2cols-0shapes-0keys-0locks.lp
%grid-2rows-2cols-1shapes-1keys-0locks.lp
%grid-2rows-2cols-1shapes-1keys-1locks.lp
%grid-2rows-2cols-1shapes-1keys-2locks.lp
%grid-2rows-2cols-1shapes-2keys-0locks.lp
%grid-2rows-2cols-1shapes-2keys-1locks.lp   
%grid-2rows-2cols-1shapes-3keys-3locks.lp
%grid-2rows-2cols-2shapes-2keys-0locks.lp
%grid-2rows-2cols-2shapes-2keys-1locks.lp
%grid-2rows-2cols-2shapes-2keys-1locks_V2.lp

\begin{SaveVerbatim}[gobble=2,commandchars=\\\{\}]{o2d:grid}
  \textcolor{bcolor}{% objects and types}
  robot(r). key(k0). key(k1). 
  cell(c1_1). 
  blackcell(c1_4). 
  ...
  \textcolor{bcolor}{% relations}
  overlap(k0,c1_2). 
  overlap(r,c2_3). 
  below(c3_1,c_2_1). 
  ...
  \textcolor{bcolor}{% shapes}
  shape(k0,heart). 
  shape(c1_4,circle). 
  shape(k1,circle). 
  shape(c3_3,heart). 
  ...
\end{SaveVerbatim}

\begin{figure}[t]
  \centering
  \begin{tabular}{c@{\quad}c}
    \scalebox{1.5}{
      \begin{tikzpicture}[background rectangle/.style={fill=gray!85}, show background rectangle]
        \draw[step=0.5cm,line width=0.3mm, color=white] (-1,-1) grid (1,1);
        \node[inner sep=0pt] (0) at (-0.25,+0.75)
            {\includegraphics[width=.02\textwidth]{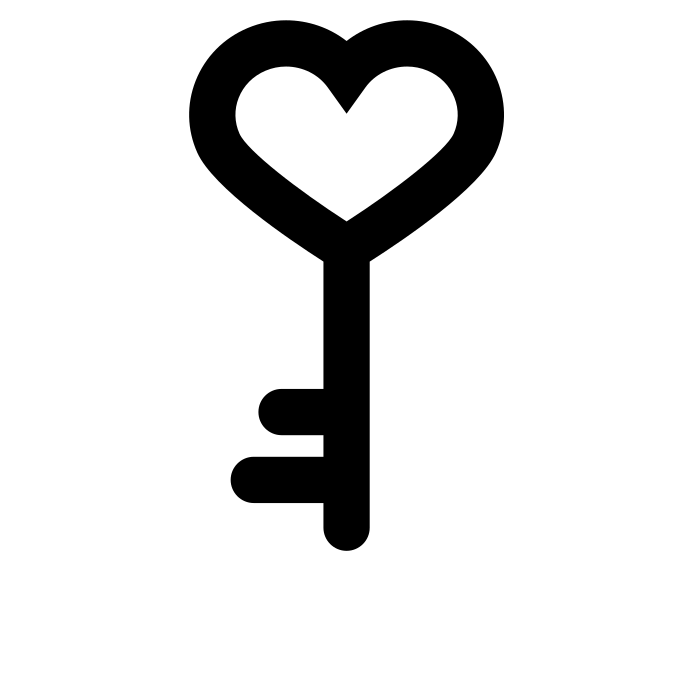}};
        \node[inner sep=0pt] (0) at (+0.25,+0.75)
            {\includegraphics[width=.02\textwidth]{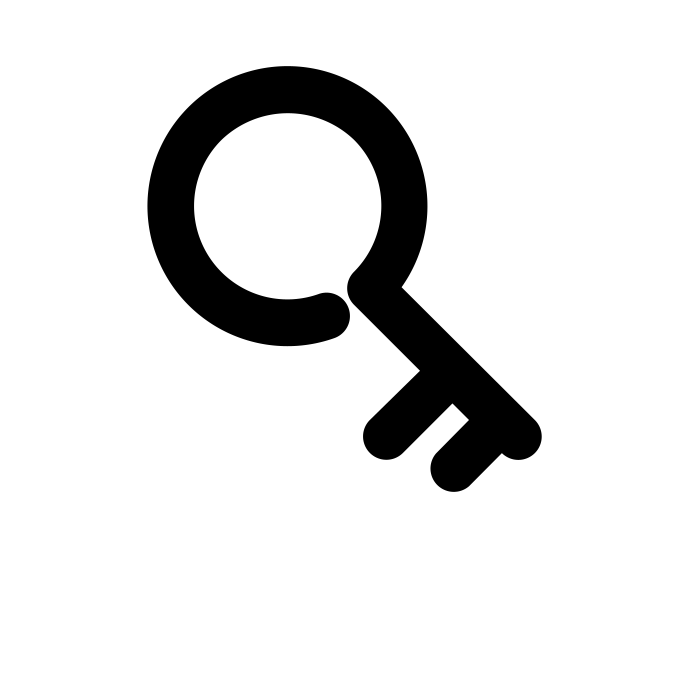}};
        \node[inner sep=0pt] (0) at (+0.75,+0.695)
            {\includegraphics[width=.0325\textwidth]{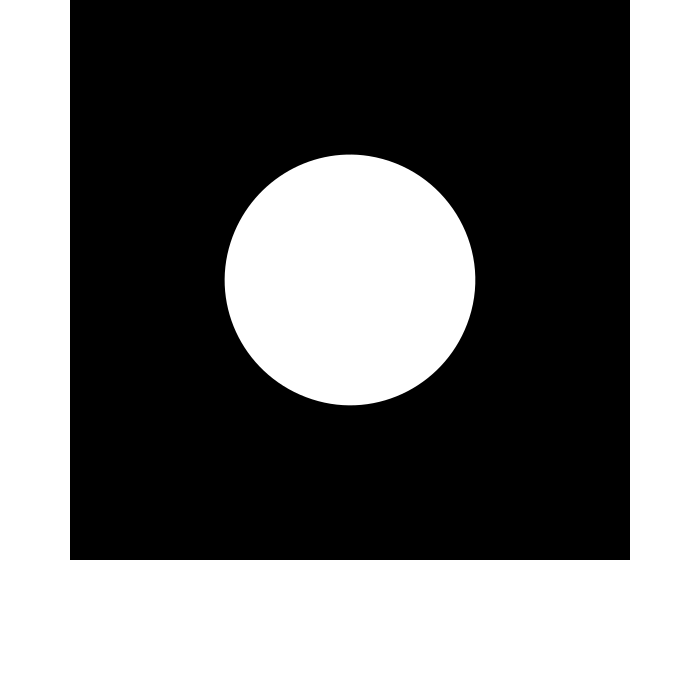}};
        \node[inner sep=0pt] (0) at (+0.25,-0.31)
            {\includegraphics[width=.0325\textwidth]{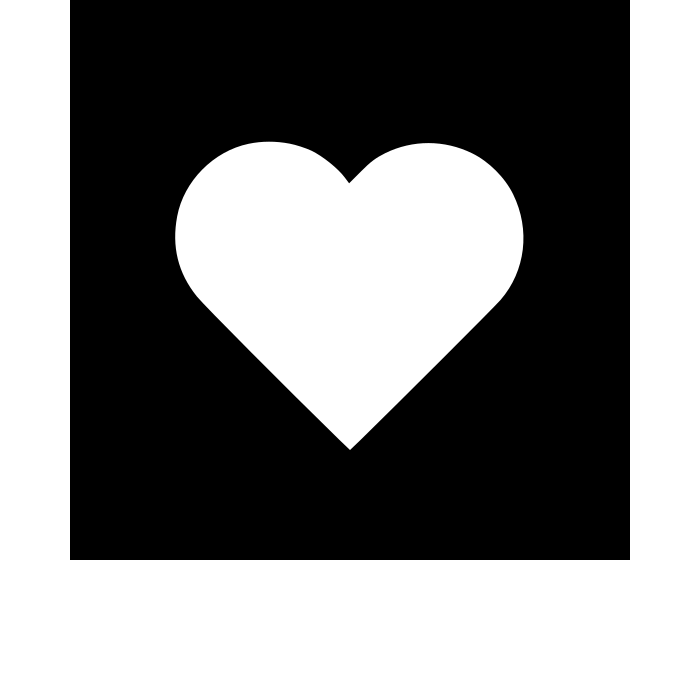}};
        \node[inner sep=0pt] (0) at (+0.25,+0.225)
            {\includegraphics[width=.023\textwidth]{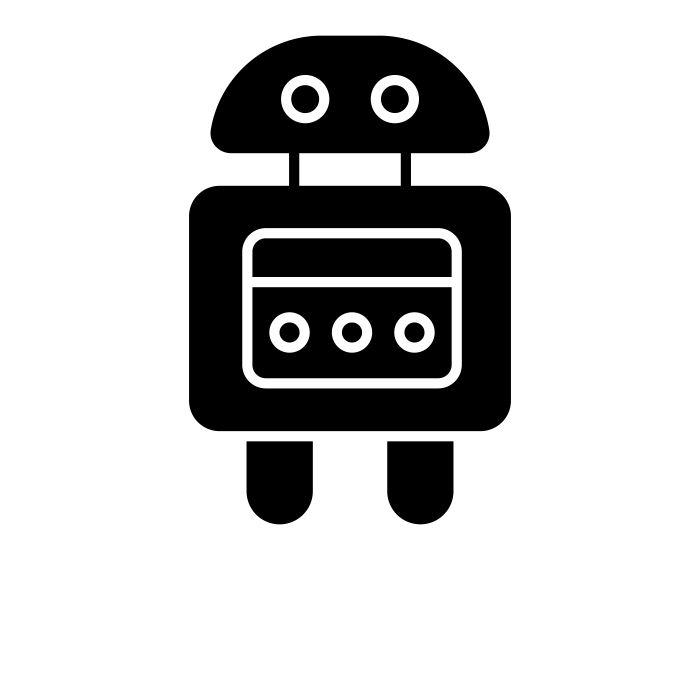}};
      \end{tikzpicture}
    }
    &
    \resizebox{0.22\textwidth}{!}{\BUseVerbatim{o2d:grid}} 
  \end{tabular}
  \vskip -.5em
  \caption{Grid scene and corresponding O2D state.}
  \label{fig:grid_appendix}
\end{figure}

\paragraph{Sokoban1 and Sokoban2.}
A puzzle where a player (Sokoban) pushes boxes (crates) around in a warehouse
(represented as a grid), trying to get them to designated storage locations.
Instances are parametrized as $(r,c,b)$ for the number of rows $r$, the number
of columns $c$, and the number of boxes $b$ located on the grid, yet the parameter
does not determine the instance because the crates and Sokoban may be in different
locations initally and at the goal.
Two different datasets are considered, one with many but smaller instances,
and the other with fewer but bigger instances.
The dataset for Sokoban1 consits of 94 instances for $(r,c,b)$ in
$\{ (1,5,b), (2,3,b), (3,2,b), (5,1,b) \}$ for $b=0,1,2$, and one extra
(larger) instance for $(4,5,2)$.
The dataset for Sokoban2 consits of 24 instances for
$(r,c,b)$ in $\{(r,5,b)\,|\,r\in\{1,2,4,5\}\}\cup\{(5,c,b)\,|\,c\in\{1,\ldots,5\}\}$
with $b=0,1,2$.
The action space for Sokoban has 8 labels: MoveUp, MoveRight, MoveDown, MoveLeft,
PushUp, PushRight, PushDown, and PushLeft.

\subsection{Mapping STRIPS States into O2D States}

For each planning instance in the data pool expressed as a pair of domain and instance
PDDL files, the full reachable state space is enumerated. Then, for each reachable
state $\bar{s}$, a ``rendering''  function $g(\cdot)$ that maps STRIPS to O2D states is applied.
The rendering function is specified by a set of DATALOG  rules that say how the visual elements
of the O2D scene are obtained.
% , resembling the work of DL-based vision modules for image analysis.
% The rendering function is specified as a set of DATALOG rules that are applied until a fixpoint is reached.
The rules used, in JSON format, are shown in Figure~\ref{fig:registry}.
They contain all the information about how planning states are mapped into scene representations
in O2D.
% As it can be seen, the rules for each domain
% generate a \emph{plausible qualitative representation} obtained from the
% output of a vision module.

% HG: I removed references to vision; we shouldn't emphasize that too much.
% When we have the mapping from images to O2D states, we'll do :-)

\begin{figure*}[t]
  %g_primitive_role_names = [ 'left', 'below', 'overlap', 'smaller', 'shape' ] # CHECK: all?
  %g_primitive_concept_names = [ 'robot', 'block', 'table', 'sokoban', 'crate', 'key', 'tile',      'cell', 'blackcell' ] # CHECK: all?
  \begin{Verbatim}[gobble=4,frame=single,numbers=left,codes={\catcode`$=3},fontsize=\relsize{-1}]
    \{ "blocks3ops" :
        \{ "constants" : ["rectangle", "t"],
          "facts"     : [ ["table",["t"]] ],
          "rules"     : \{ "block"      : [ ["block(X)",           ["ontable(X)"]],
                                           ["block(X)",           ["on(X,Y)"]] ],
                          "below"      : [ ["below(X,Y)",         ["on(Y,X)"]],
                                           ["below(X,Y)",         ["ontable(Y)", "table(X)"]] ],
                          "smaller"    : [ ["smaller(X,Y)",       ["block(X)", "table(Y)"]],
                                           ["smaller(X,Z)",       ["smaller(X,Y)", "smaller(Y,Z)"]] ],
                          "shape"      : [ ["shape(X,rectangle)", ["object(X)"]] ]
                        \} \},
       "blocks4ops" :
        \{ "constants" : ["rectangle", "r", "t"],
          "facts"     : [ ["robot",["r"]], ["table",["t"]] ],
          "rules"     : \{ "block"      : [ ["block(X)",           ["ontable(X)"]],
                                           ["block(X)",           ["on(X,Y)"]] ],
                          "overlap"    : [ ["overlap(X,Y)",       ["holding(X)", "robot(Y)"]],
                                           ["overlap(Y,X)",       ["overlap(X,Y)"]] ],
                          "below"      : [ ["below(X,Y)",         ["on(Y,X)"]],
                                           ["below(X,Y)",         ["ontable(Y)", "table(X)"]] ],
                          "smaller"    : [ ["smaller(X,Y)",       ["block(X)", "table(Y)"]],
                                           ["smaller(X,Y)",       ["block(X)", "robot(Y)"]],
                                           ["smaller(X,Y)",       ["robot(X)", "table(Y)"]],
                                           ["smaller(X,Z)",       ["smaller(X,Y)", "smaller(Y,Z)"]] ],
                          "shape"      : [ ["shape(X,rectangle)", ["object(X)"]] ]
                        \} \},
      "hanoi1op" :
        \{ "constants" : ["rectangle"],
          "facts"     : [],
          "rules"     : \{ "overlap"    : [ ["overlap(X,Y)",       ["disk(X)", "peg(Y)", "on(X,Y)"]],
                                           ["overlap(Y,X)",       ["overlap(X,Y)"]] ],
                          "below"      : [ ["below(X,Y)",         ["on(Y,X)", "disk(X)", "disk(Y)"]],
                                           ["below(X,Y)",         ["on(Y,X)", "peg(X)", "disk(Y)"]] ],
                          "shape"      : [ ["shape(X,rectangle)", ["object(X)"]] ]
                        \} \},
      "hanoi4ops" : \{ "defer-to" : "hanoi1op" \},
      "slidingtile" :
        \{ "constants" : ["rectangle"],
          "facts"     : [],
          "rules"     : \{ "cell"       : [ ["cell(X)",            ["position(X)"]] ],
                          "overlap"    : [ ["overlap(X,Y)",       ["at(X,Y)"]],
                                           ["overlap(Y,X)",       ["overlap(X,Y)"]] ],
                          "shape"      : [ ["shape(X,rectangle)", ["object(X)"]] ]
                        \} \},
      "grid" :
        \{ "constants" : ["r"],
          "facts"     : [ ["robot",["r"]] ],
          "rules"     : \{ "cell"       : [ ["cell(X)",            ["place(X)", "open(X)"]] ],
                          "blackcell"  : [ ["blackcell(X)",       ["locked(X)"]] ],
                          "overlap"    : [ ["overlap(X,r)",       ["at_robot(X)"]],
                                           ["overlap(X,Y)",       ["at(X,Y)"]],
                                           ["overlap(Y,X)",       ["overlap(X,Y)"]] ],
                          "smaller"    : [ ["smaller(X,Y)",       ["robot(X)", "place(Y)"]],
                                           ["smaller(X,Y)",       ["key(X)", "place(Y)"]],
                                           ["smaller(X,Z)",       ["smaller(X,Y)", "smaller(Y,Z)"]] ],
                          "shape"      : [ ["shape(X,S)",         ["lock_shape(X,S)"]],
                                           ["shape(X,S)",         ["key_shape(X,S)"]] ]
                        \} \},
      "sokoban1" :
        \{ "constants" : ["sokoban1", "rectangle", "sokoshape"],
          "facts"     : [ ["sokoban",["sokoban1"]] ],
          "rules"     : \{ "cell"       : [ ["cell(X)",            ["leftof(X,Y)"]],
                                           ["cell(Y)",            ["leftof(X,Y)"]],
                                           ["cell(X)",            ["below(X,Y)"]],
                                           ["cell(Y)",            ["below(X,Y)"]],
                                           ["cell(Y)",            ["at(X,Y)"]] ],
                          "overlap"    : [ ["overlap(X,Y)",       ["at(X,Y)"]],
                                           ["overlap(Y,X)",       ["overlap(X,Y)"]] ],
                          "left"       : [ ["left(X,Y)",          ["leftof(X,Y)"]] ],
                          "shape"      : [ ["shape(X,sokoshape)", ["sokoban(X)"]],
                                           ["shape(X,rectangle)", ["cell(X)"]],
                                           ["shape(X,rectangle)", ["crate(X)"]] ]
                        \} \},
      "sokoban2" : \{ "defer-to" : "sokoban1" \}
    \}
  \end{Verbatim}
  \vskip -1eM
  \caption{JSON specification of the O2D rendering function $g(\cdot)$ that is used to generate the datasets from the STRIPS instances.
    For each reachable STRIPS state $\bar{s}$ in such an instance, the ``rules'' are applied until reaching a fixpoint,
    where a rules consists of a head (single O2D atom) and a body (list of atoms).
    The resulting O2D state is obtained by preserving the resulting O2D atoms, and removing all the STRIPS
    atoms except the static ones that specify types;
    i.e., \texttt{robot}, \texttt{block}, \texttt{table}, \texttt{sokoban}, \texttt{crate}, \texttt{key}, and \texttt{tile}.
  }
  \label{fig:registry}
\end{figure*}

\subsection{Learned Grounded Domains}

Figures~\ref{fig:domain:blocks3ops}--\ref{fig:domain:sokoban} show the learned grounded
domains for all the learning tasks in the experiments. In each case, we show the final
and optimal value of the (optimized) cost function (lines 357--367 in Fig.~\ref{fig:program:4}),
the grounded predicates from the pool $\P$ that make up the learned domain, the action
schemas in the domain, the constants (if any), and some stats about the incremental solver.

In these models, grounded predicates $P$ are directly represented by their
grounding $\sigma(P)$ as computed by the learner, and their description
is given as \texttt{<predicate>/<arity>} in the top part of each domain.
The notation to represent concepts, roles and predicates is as follows,
where $d(\cdot)$ is the denotation function:
\begin{enumerate}[--]
  \item O2D concept $C$ and role $R$ denoted by $C$ and $R$, resp.,
  \item Concept `$\top$' denoted by \texttt{Top},
  \item Concept `$\exists\,R.C$' denoted by \texttt{ER[$d(R)$,$d(C)$]},
  \item Concept `$C\sqcap C'$' denoted by \texttt{INTER[$d(C)$,$d(C')$]},
  \item Role `$R^{-1}$' denoted by \texttt{INV[$d(R)$]},
  \item Role `$R\circ R'$' denoted by \texttt{COMP[$d(R)$,$d(R')$]}, and
  \item Predicate `$C\sqsubseteq C'$' denoted by \texttt{SUBSET[$d(C)$,$d(C')$]}.
\end{enumerate}
Recall that concepts and roles directly yield predicates of arity 1 and 2 respectively,
while nullary predicates are obtained with the subset construction (i.e., $C\sqsubseteq C'$).

\begin{figure*}[t]
  \begin{Verbatim}[gobble=4,frame=single,numbers=left,codes={\catcode`$=3},fontsize=\relsize{-1}]
    \textcolor{bcolor}{Optimization:} (10,5,0,10,8)
    \textcolor{bcolor}{1 constant(s):} t
    \textcolor{bcolor}{2 predicate(s):} ER[below,Top]/1, INV[below]/2

    \textcolor{acolor}{Stack(1,2):}
      \textcolor{pcolor}{pre:} $\neg$ER[below,Top](1), INV[below](1,t), $\neg$ER[below,Top](2)
      \textcolor{pcolor}{eff:} ER[below,Top](2), $\neg$INV[below](1,t), INV[below](1,2)
    \textcolor{acolor}{Newtower(1,2):}
      \textcolor{pcolor}{pre:} $\neg$ER[below,Top](1), INV[below](1,2)
      \textcolor{pcolor}{eff:} $\neg$ER[below,Top](2), INV[below](1,t), $\neg$INV[below](1,2)
    \textcolor{acolor}{Move(1,2,3):}
      \textcolor{pcolor}{pre:} $\neg$ER[below,Top](1), INV[below](1,2), $\neg$ER[below,Top](3)
      \textcolor{pcolor}{eff:} $\neg$ER[below,Top](2), ER[below,Top](3), $\neg$INV[below](1,2), INV[below](1,3)

    \textcolor{bcolor}{#calls=5, solve_wall_time=1.97, solve_ground_time=1.92, verify_time=0.84, elapsed_time=2.97}
  \end{Verbatim}
  \vskip -1eM
  \caption{Learned grounded domain for Blocks3ops}
  \label{fig:domain:blocks3ops}
\end{figure*}

\begin{figure*}[t]
  \begin{Verbatim}[gobble=4,frame=single,numbers=left,codes={\catcode`$=3},fontsize=\relsize{-1}]
    \textcolor{bcolor}{Optimization:} (10,7,0,14,9)
    \textcolor{bcolor}{2 constant(s):} r, t
    \textcolor{bcolor}{3 predicate(s):} ER[overlap,Top]/1, INTER[ER[below,Top],block]/1, below/2

    \textcolor{acolor}{Pickup(1):}
      \textcolor{pcolor}{pre:} $\neg$ER[overlap,Top](r), $\neg$INTER[ER[below,Top],block](1), below(t,1)
      \textcolor{pcolor}{eff:} ER[overlap,Top](r), ER[overlap,Top](1), $\neg$below(t,1)
    \textcolor{acolor}{Putdown(1):}
      \textcolor{pcolor}{pre:} ER[overlap,Top](1)
      \textcolor{pcolor}{eff:} $\neg$ER[overlap,Top](r), $\neg$ER[overlap,Top](1), below(t,1)
    \textcolor{acolor}{Unstack(1,2):}
      \textcolor{pcolor}{pre:} $\neg$ER[overlap,Top](r), $\neg$INTER[ER[below,Top],block](1), below(2,1)
      \textcolor{pcolor}{eff:} ER[overlap,Top](r), ER[overlap,Top](1), $\neg$INTER[ER[below,Top],block](2), $\neg$below(2,1)
    \textcolor{acolor}{Stack(1,2):}
      \textcolor{pcolor}{pre:} ER[overlap,Top](1), $\neg$INTER[ER[below,Top],block](2)
      \textcolor{pcolor}{eff:} $\neg$ER[overlap,Top](r), $\neg$ER[overlap,Top](1), INTER[ER[below,Top],block](2), below(2,1)

    \textcolor{bcolor}{#calls=7, solve_wall_time=23.66, solve_ground_time=23.37, verify_time=29.42, elapsed_time=53.70}
  \end{Verbatim}
  \vskip -1eM
  \caption{Learned grounded domain for Blocks4ops}
  \label{fig:domain:blocks4ops}
\end{figure*}

\begin{figure*}[t]
  \begin{Verbatim}[gobble=4,frame=single,numbers=left,codes={\catcode`$=3},fontsize=\relsize{-1}]
    \textcolor{bcolor}{Optimization:} (4,5,3,4,4)
    \textcolor{bcolor}{3 predicate(s):} ER[below,Top]/1, below/2, smaller/2
    \textcolor{bcolor}{1 static predicate(s):} smaller/2

    \textcolor{acolor}{Move(1,2,3):}
      \textcolor{pcolor}{static:} smaller(1,2)
      \textcolor{pcolor}{pre:} $\neg$ER[below,Top](1), $\neg$ER[below,Top](2), below(3,1)
      \textcolor{pcolor}{eff:} ER[below,Top](2), $\neg$ER[below,Top](3), below(2,1), $\neg$below(3,1)

    \textcolor{bcolor}{#calls=4, solve_wall_time=1.59, solve_ground_time=1.53, verify_time=0.44, elapsed_time=2.16}
  \end{Verbatim}
  \vskip -1eM
  \caption{Learned grounded domain for Hanoi1op}
  \label{fig:domain:hanoi1op}
\end{figure*}

\begin{figure*}[t]
  \begin{Verbatim}[gobble=4,frame=single,numbers=left,codes={\catcode`$=3},fontsize=\relsize{-1}]
    \textcolor{bcolor}{Optimization:} (16,5,5,16,22)
    \textcolor{bcolor}{4 predicate(s):} ER[below,Top]/1, ER[smaller,Top]/1, INV[below]/2, INV[smaller]/2
    \textcolor{bcolor}{2 static predicate(s):} ER[smaller,Top]/1, INV[smaller]/2

    \textcolor{acolor}{MoveFromPegToPeg(1,2,3):}
      \textcolor{pcolor}{static:} $\neg$ER[smaller,Top](1), $\neg$INV[smaller](1,3)
      \textcolor{pcolor}{pre:} $\neg$ER[below,Top](2), INV[below](2,1), $\neg$ER[below,Top](3)
      \textcolor{pcolor}{eff:} $\neg$ER[below,Top](1), ER[below,Top](3), $\neg$INV[below](2,1), INV[below](2,3)
    \textcolor{acolor}{MoveFromPegToDisk(1,2,3):}
      \textcolor{pcolor}{static:} $\neg$INV[smaller](1,2), INV[smaller](3,2), $\neg$ER[smaller,Top](3)
      \textcolor{pcolor}{pre:} $\neg$ER[below,Top](1), $\neg$ER[below,Top](2), INV[below](1,3)
      \textcolor{pcolor}{eff:} ER[below,Top](2), $\neg$ER[below,Top](3), INV[below](1,2), $\neg$INV[below](1,3)
    \textcolor{acolor}{MoveFromDiskToPeg(1,2,3):}
      \textcolor{pcolor}{static:} ER[smaller,Top](2), $\neg$ER[smaller,Top](3)
      \textcolor{pcolor}{pre:} $\neg$ER[below,Top](1), INV[below](1,2), $\neg$ER[below,Top](3)
      \textcolor{pcolor}{eff:} $\neg$ER[below,Top](2), ER[below,Top](3), $\neg$INV[below](1,2), INV[below](1,3)
    \textcolor{acolor}{MoveFromDiskToDisk(1,2,3):}
      \textcolor{pcolor}{static:} $\neg$INV[smaller](1,2), ER[smaller,Top](2), ER[smaller,Top](3)
      \textcolor{pcolor}{pre:} $\neg$ER[below,Top](1), $\neg$ER[below,Top](2), INV[below](1,3)
      \textcolor{pcolor}{eff:} ER[below,Top](2), $\neg$ER[below,Top](3), INV[below](1,2), $\neg$INV[below](1,3)

    \textcolor{bcolor}{#calls=6, solve_wall_time=14.23, solve_ground_time=12.67, verify_time=0.59, elapsed_time=15.06}
  \end{Verbatim}
  \vskip -1eM
  \caption{Learned grounded domain for Hanoi4ops}
  \label{fig:domain:hanoi4ops}
\end{figure*}

\begin{figure*}[t]
  \begin{Verbatim}[gobble=4,frame=single,numbers=left,codes={\catcode`$=3},fontsize=\relsize{-1}]
    \textcolor{bcolor}{Optimization:} (16,5,6,24,12)
    \textcolor{bcolor}{4 predicate(s):} ER[overlap,Top]/1, overlap/2, INV[left]/2, INV[below]/2
    \textcolor{bcolor}{2 static predicate(s):} INV[left]/2, INV[below]/2

    \textcolor{acolor}{MoveUp(1,2,3):}
      \textcolor{pcolor}{static:} INV[below](3,2)
      \textcolor{pcolor}{pre:} overlap(1,2), $\neg$ER[overlap,Top](3)
      \textcolor{pcolor}{eff:} $\neg$ER[overlap,Top](2), ER[overlap,Top](3), $\neg$overlap(1,2), overlap(1,3), $\neg$overlap(2,1), overlap(3,1)
    \textcolor{acolor}{MoveRight(1,2,3):}
      \textcolor{pcolor}{static:} INV[left](3,2)
      \textcolor{pcolor}{pre:} overlap(1,2), $\neg$ER[overlap,Top](3)
      \textcolor{pcolor}{eff:} $\neg$ER[overlap,Top](2), ER[overlap,Top](3), $\neg$overlap(1,2), overlap(1,3), $\neg$overlap(2,1), overlap(3,1)
    \textcolor{acolor}{MoveDown(1,2,3):}
      \textcolor{pcolor}{static:} INV[below](2,3)
      \textcolor{pcolor}{pre:} overlap(2,1), $\neg$ER[overlap,Top](3)
      \textcolor{pcolor}{eff:} $\neg$ER[overlap,Top](2), ER[overlap,Top](3), $\neg$overlap(1,2), overlap(1,3), $\neg$overlap(2,1), overlap(3,1)
    \textcolor{acolor}{MoveLeft(1,2,3):}
      \textcolor{pcolor}{static:} INV[left](2,3)
      \textcolor{pcolor}{pre:} overlap(2,1), $\neg$ER[overlap,Top](3)
      \textcolor{pcolor}{eff:} $\neg$ER[overlap,Top](2), ER[overlap,Top](3), $\neg$overlap(1,2), overlap(1,3), $\neg$overlap(2,1), overlap(3,1)

    \textcolor{bcolor}{#calls=6, solve_wall_time=3.00, solve_ground_time=2.89, verify_time=1.20, elapsed_time=4.43}
  \end{Verbatim}
  \vskip -1eM
  \caption{Learned grounded domain for Sliding Tile}
  \label{fig:domain:slidingtile}
\end{figure*}

\begin{figure*}[t]
  \begin{Verbatim}[gobble=4,frame=single,numbers=left,codes={\catcode`$=3},fontsize=\relsize{-1}]
    \textcolor{bcolor}{Optimization:} (34,8,11,28,42)
    \textcolor{bcolor}{1 constant(s):} r
    \textcolor{bcolor}{8 predicate(s):} SUBSET[key,ER[overlap,Top]]/0, cell/1, ER[smaller,Top]/1, INTER[key,ER[overlap,Top]]/1,
                    below/2, overlap/2, INV[left]/2, COMP[shape,INV[shape]]/2
    \textcolor{bcolor}{4 static predicate(s):} ER[smaller,Top]/1, below/2, INV[left], COMP[shape,INV[shape]]/2

    \textcolor{acolor}{MoveUp(1,2):}
      \textcolor{pcolor}{static:} below(1,2)
      \textcolor{pcolor}{pre:} overlap(r,1), cell(2)
      \textcolor{pcolor}{eff:} $\neg$overlap(r,1), overlap(r,2), $\neg$overlap(1,r), overlap(2,r)
    \textcolor{acolor}{MoveRight(1,2):}
      \textcolor{pcolor}{static:} INV[left](2,1)
      \textcolor{pcolor}{pre:} overlap(1,r), cell(2)
      \textcolor{pcolor}{eff:} $\neg$overlap(r,1), overlap(r,2), $\neg$overlap(1,r), overlap(2,r)
    \textcolor{acolor}{MoveDown(1,2):}
      \textcolor{pcolor}{static:} below(2,1)
      \textcolor{pcolor}{pre:} overlap(1,r), cell(2)
      \textcolor{pcolor}{eff:} $\neg$overlap(r,1), overlap(r,2), $\neg$overlap(1,r), overlap(2,r)
    \textcolor{acolor}{MoveLeft(1,2):}
      \textcolor{pcolor}{static:} INV[left](1,2)
      \textcolor{pcolor}{pre:} overlap(r,1), cell(2)
      \textcolor{pcolor}{eff:} $\neg$overlap(r,1), overlap(r,2), $\neg$overlap(1,r), overlap(2,r)
    \textcolor{acolor}{Pickup(1,2):}
      \textcolor{pcolor}{pre:} SUBSET[key,ER[overlap,Top]], overlap(1,r), overlap(1,2)
      \textcolor{pcolor}{eff:} $\neg$SUBSET[key,ER[overlap,Top]], $\neg$INTER[key,ER[overlap,Top]](2), $\neg$overlap(1,2), $\neg$overlap(2,1)
    \textcolor{acolor}{Putdown(1,2):}
      \textcolor{pcolor}{static:} ER[smaller,Top](2)
      \textcolor{pcolor}{pre:} overlap(1,r), $\neg$INTER[key,ER[overlap,Top]](2)
      \textcolor{pcolor}{eff:} SUBSET[key,ER[overlap,Top]], INTER[key,ER[overlap,Top]](2), overlap(1,2), overlap(2,1)
    \textcolor{acolor}{UnlockFromAbove(1,2,3):}
      \textcolor{pcolor}{static:} below(2,1), COMP[shape,INV[shape]](2,3)
      \textcolor{pcolor}{pre:} $\neg$SUBSET[key,ER[overlap,Top]], overlap(1,r), $\neg$cell(2), $\neg$INTER[key,ER[overlap,Top]](3)
      \textcolor{pcolor}{eff:} cell(2)
    \textcolor{acolor}{UnlockFromRight(1,2,3):}
      \textcolor{pcolor}{static:} INV[left](2,1), COMP[shape,INV[shape]](1,3), ER[smaller,Top](3)
      \textcolor{pcolor}{pre:} $\neg$cell(1), overlap(r,2), $\neg$INTER[key,ER[overlap,Top]](3)
      \textcolor{pcolor}{eff:} cell(1)
    \textcolor{acolor}{UnlockFromBelow(1,2,3):}
      \textcolor{pcolor}{static:} below(2,1), COMP[shape,INV[shape]](1,3)
      \textcolor{pcolor}{pre:} $\neg$SUBSET[key,ER[overlap,Top]], $\neg$cell(1), overlap(2,r), $\neg$INTER[key,ER[overlap,Top]](3)
      \textcolor{pcolor}{eff:} cell(1)
    \textcolor{acolor}{UnlockFromLeft(1,2,3):}
      \textcolor{pcolor}{static:} INV[left](1,2), COMP[shape,INV[shape]](3,1), ER[smaller,Top](3)
      \textcolor{pcolor}{pre:} $\neg$cell(1), overlap(r,2), $\neg$INTER[key,ER[overlap,Top]](3)
      \textcolor{pcolor}{eff:} cell(1)

    \textcolor{bcolor}{#calls=27, solve_wall_time=4229.67, solve_ground_time=3536.23, verify_time=2404.87, elapsed_time=6653.03}
  \end{Verbatim}
  \vskip -1eM
  \caption{Learned grounded domain for IPC Grid}
  \label{fig:domain:grid:new}
\end{figure*}

\begin{figure*}[t]
  \begin{Verbatim}[gobble=4,frame=single,numbers=left,codes={\catcode`$=3},fontsize=\relsize{-1}]
    \textcolor{bcolor}{Optimization:} (32,5,6,64,32)
    \textcolor{bcolor}{1 constant(s):} sokoban1
    \textcolor{bcolor}{4 predicate(s):} ER[overlap,Top]/1, below/2, overlap/2, INV[left]/2
    \textcolor{bcolor}{2 static predicate(s):} below/2, INV[left]/2

    \textcolor{acolor}{MoveUp(1,2):}
      \textcolor{pcolor}{static:} below(1,2)
      \textcolor{pcolor}{pre:} overlap(1,sokoban1), $\neg$ER[overlap,Top](2)
      \textcolor{pcolor}{eff:} $\neg$ER[overlap,Top](1), ER[overlap,Top](2), $\neg$overlap(sokoban1,1), overlap(sokoban1,2), $\neg$overlap(1,sokoban1),
           overlap(2,sokoban1)
    \textcolor{acolor}{MoveRight(1,2):}
      \textcolor{pcolor}{static:} INV[left](2,1)
      \textcolor{pcolor}{pre:} overlap(1,sokoban1), $\neg$ER[overlap,Top](2)
      \textcolor{pcolor}{eff:} $\neg$ER[overlap,Top](1), ER[overlap,Top](2), $\neg$overlap(sokoban1,1), overlap(sokoban1,2), $\neg$overlap(1,sokoban1),
           overlap(2,sokoban1)
    \textcolor{acolor}{MoveDown(1,2):}
      \textcolor{pcolor}{static:} below(2,1)
      \textcolor{pcolor}{pre:} overlap(sokoban1,1), $\neg$ER[overlap,Top](2)
      \textcolor{pcolor}{eff:} $\neg$ER[overlap,Top](1), ER[overlap,Top](2), $\neg$overlap(sokoban1,1), overlap(sokoban1,2), $\neg$overlap(1,sokoban1),
           overlap(2,sokoban1)
    \textcolor{acolor}{MoveLeft(1,2):}
      \textcolor{pcolor}{static:} INV[left](1,2)
      \textcolor{pcolor}{pre:} overlap(1,sokoban1), $\neg$ER[overlap,Top](2)
      \textcolor{pcolor}{eff:} $\neg$ER[overlap,Top](1), ER[overlap,Top](2), $\neg$overlap(sokoban1,1), overlap(sokoban1,2), $\neg$overlap(1,sokoban1),
           overlap(2,sokoban1)
    \textcolor{acolor}{PushUp(1,2,3,4):}
      \textcolor{pcolor}{static:} below(3,1), below(1,4)
      \textcolor{pcolor}{pre:} overlap(1,2), overlap(3,sokoban1), $\neg$ER[overlap,Top](4)
      \textcolor{pcolor}{eff:} $\neg$ER[overlap,Top](3), ER[overlap,Top](4), overlap(sokoban1,1), $\neg$overlap(sokoban1,3),
           overlap(1,sokoban1), $\neg$overlap(3,sokoban1), $\neg$overlap(1,2), $\neg$overlap(2,1), overlap(2,4), overlap(4,2)
    \textcolor{acolor}{PushRight(1,2,3,4):}
      \textcolor{pcolor}{static:} INV[left](1,3), INV[left](4,1)
      \textcolor{pcolor}{pre:} overlap(2,1), overlap(sokoban1,3), $\neg$ER[overlap,Top](4)
      \textcolor{pcolor}{eff:} $\neg$ER[overlap,Top](3), ER[overlap,Top](4), overlap(sokoban1,1), $\neg$overlap(sokoban1,3),
           overlap(1,sokoban1), $\neg$overlap(3,sokoban1), $\neg$overlap(1,2), $\neg$overlap(2,1), overlap(2,4), overlap(4,2)
    \textcolor{acolor}{PushDown(1,2,3,4):}
      \textcolor{pcolor}{static:} below(1,3), below(4,1)
      \textcolor{pcolor}{pre:} overlap(2,1), overlap(sokoban1,3), $\neg$ER[overlap,Top](4)
      \textcolor{pcolor}{eff:} $\neg$ER[overlap,Top](3), ER[overlap,Top](4), overlap(sokoban1,1), $\neg$overlap(sokoban1,3),
           overlap(1,sokoban1), $\neg$overlap(3,sokoban1), $\neg$overlap(1,2), $\neg$overlap(2,1), overlap(2,4), overlap(4,2)
    \textcolor{acolor}{PushLeft(1,2,3,4):}
      \textcolor{pcolor}{static:} INV[left](3,1), INV[left](1,4)
      \textcolor{pcolor}{pre:} overlap(2,1), overlap(3,sokoban1), $\neg$ER[overlap,Top](4)
      \textcolor{pcolor}{eff:} $\neg$ER[overlap,Top](3), ER[overlap,Top](4), overlap(sokoban1,1), $\neg$overlap(sokoban1,3),
           overlap(1,sokoban1), $\neg$overlap(3,sokoban1), $\neg$$\neg$overlap(1,2), overlap(2,1), overlap(2,4), overlap(4,2)

    \textcolor{bcolor}{#calls=11, solve_wall_time=12565.02, solve_ground_time=5314.35, verify_time=165.19, elapsed_time=12740.43}
  \end{Verbatim}
  \vskip -1eM
  \caption{Learned grounded domain for Sokoban}
  \label{fig:domain:sokoban}
\end{figure*}

\subsection{Implementation Details}

\subsubsection{Full ASP Program}

The code in ASP for learning the instances $P_i{=}\tup{D,I_i}$ from multiple input
graphs $G_i$, using a pool of predicates $\P$,
is shown in Figures~\ref{fig:program:1}--\ref{fig:program:4}.
Each graph $G_i$ is assumed to be encoded using the atoms \atom|node(I,S)| and
\atom|tlabel(I,T,L)| where \atom|S| and \atom|T=(S1,S2)| denote nodes and transitions
in the graph $G_i$ with index \atom|I|, and \atom|L| denotes the corresponding
action label.

At each step of incremental learning, each instance \atom|I| and node \atom|S|
from this instance that must be taken into account at that step is marked with
an atom \atom|relevant(I,S)|.
%Instances that are not relevant for this step are marked with \atom|disabled(I)|.
Truth values \atom|V| of ground atoms \atom|(P,OO)| from $\P$,
in the state \atom|S| of instance \atom|I|, are encoded in the input with atoms
\atom|val(I,(P,OO),S,V)|.
When computing the truth values for such atoms, \textit{redundant} predicates from
$\P$ are pruned. A predicate is redundant when its denotation over all states
in the dataset is the same as some previously considered predicate; i.e., given some
enumeration of the predicates in the pool, the predicate $p_i$ is redundant iff there
is $p_j$ such that for each state $s$ of each instance, $p^s_i = p^s_j$, $j\,{<}\,i$.
Moreover, for each instance \atom|I|, we compute the predicates that are static over
that instance, mark them with the fact \atom|f\_static(I,P)|, and encode their truth
value \atom|V| in all states of that instance with a single fact \atom|val(I,(P,OO2),V)|.
Each predicate \atom|P| in $\P$ is marked with the fact \atom|feature(P)|, and their
arity \atom|N| is encoded by fact \atom|f\_arity(P,N)|.
%Finally, for each domain, each \textit{object type} predicate \atom|P| that has a
%non-empty denotation over the input instances is immediately selected as a predicate
%(via an atom \atom|pred(P)|) in order to allow the learned PDDL schemas to be typed.
The number of action schemas is set to the number of action labels and the objects
are extracted from the valuation of the concept $\top$ (line 29).
The max number of chosen predicates is set to the value of the constant \atom|num\_predicates|
($12$ by default), while the max arity of actions is set to the value of the
constant \atom|max\_arity| ($3$ by default, but with max value of 4 for the code shown).

Exploiting the fact that the predicates in the pool have a maximum arity of 2,
the applicability relation for grounded actions as well as the successor function
are factored (lines 147--249). This makes the code longer but results in improved
grounding and solving times.

\subsubsection{Verifier}

The verifier, written in Python, receives the input data graph $G_\D$ for the
instance together with the learned grounded model $D$, and outputs a subset $\Delta$
of states in $G_\D$: either $\Delta\,{=}\,\emptyset$ meaning \emph{successful verification},
$\Delta\,{=}\,\{s,s'\}$ of two such states that are identical modulo the grounded
predicates in the model (cf.\ constraint C1 in \ref{def:task}), or a non-empty
subset of states in $G_\D$ for which constraint C2 does not hold.
The verifier is called from the incremental solver that then uses $\Delta$
to extend the learning dataset $\D$ as described in the paper.

% (model) planning states: s,s', ...
% O2D states: \bar s, \bar s', ...
% function h maps O2D states into planning states
% function g in Thm 4: maps planning states into O2D states
% (rendering) function r maps STRIPS states \tilde{s} into O2D states s

The verifier works as follows. First, C1 is checked for all pair of
states in $G_\D$.
If for some pair $(s,s')$ C1 fails, the verifier terminates and outputs $\{s,s'\}$.
Otherwise, for each state $s$ in $G_\D$, the verifier checks that each transition
$(s,s')$ with label $\alpha$ in $G_\D$ has a matching transition $(h(s),h(s'))$
in the learned model $D$ via a grounded action with label $\alpha$, and vice
versa, that each transition $(h(s),\bar{s}')$ in $D$ via a grounded action with
label $\alpha$ has a matching transition $(s,s'')$ in $G_\D$ with label $\alpha$
such that $h(s'')=\bar{s}'$; if some of these two checks fails, the state $s$ is
added to the output set for the verifier.

\medskip

Appendix continues with figures in the next few pages.

\begin{figure*}[t]
  \begin{Verbatim}[gobble=4,frame=single,numbers=left,codes={\catcode`$=3},fontsize=\relsize{-2}]
    \textcolor{bcolor}{% Suggested call}
    \textcolor{bcolor}{% clingo -t 6 --sat-prepro=2 --time-limit=7200 <this-solver> <graph-files>}

    \textcolor{bcolor}{% Constants and options}
    #const num_predicates = 12.
    #const max_action_arity = 3.
    #const null_arg = (null,).
    #const opt_equal_objects = 0.          \textcolor{bcolor}{% Allow same obj as argument for grounded actions}
    #const opt_allow_negative_precs = 1.   \textcolor{bcolor}{% Allow for negative preconditions}
    #const opt_fill = 1.                   \textcolor{bcolor}{% Fill in missing negative valuations for primitive predicates}
    #const opt_symmetries = 1.             \textcolor{bcolor}{% Some (simple) symmetry breaking}

    \textcolor{bcolor}{% Input Instances defined by instance/1, and graphs by tlabel/3 and node/2}
    \textcolor{bcolor}{% O2D features defined by feature/1, f_arity/2, f_static/2, fval/3, and fval/4}
    nullary(F)   :- feature(F), f_arity(F,1), 1 \{ fval(I,(F,null_arg),0..1) \}.                \textcolor{bcolor}{% Explicit zero_arity predicates}
    nullary(F)   :- feature(F), f_arity(F,1), 1 \{ fval(I,(F,null_arg),S,0..1) : node(I,S) \}.  \textcolor{bcolor}{% Explicit zero_arity predicates}
    p_arity(F,N) :- feature(F), f_arity(F,N), not nullary(F).                                 \textcolor{bcolor}{% Explicit zero_arity predicates}
    p_arity(F,0) :- nullary(F).                                                               \textcolor{bcolor}{% Explicit zero_arity predicates}
    :- p_arity(F,N), nullary(F), N > 0.                                                       \textcolor{bcolor}{% Explicit zero_arity predicates}

    \textcolor{bcolor}{% Relevant nodes (all relevant by default; overriden by incremental solver)}
    #defined filename/1.
    #defined partial/2.
    relevant(I,S) :- node(I,S), not partial(I,File) : filename(File).

    \textcolor{bcolor}{% Actions and objects (objects come from denotation of concept Top)}
    action(A) :- tlabel(I,(S,T),A), relevant(I,S).
    \{ a_arity(A,0..max_action_arity) \} = 1 :- action(A).
    object(I,O) :- fval(I,(top,(O,)),1).

    \textcolor{bcolor}{% Choose predicates from high-level language}
    \{ pred(F) : feature(F) \} num_predicates.

    \textcolor{bcolor}{% Tuples of variables/constants for lifted effects and preconditions}
    #defined constant/1.
    argtuple(null_arg,0). \textcolor{bcolor}{% Explicit zero arity predicates}
    argtuple((C1,),1)     :- constant(C1).
    argtuple((V1,),1)     :- V1 = 1..max_action_arity.
    argtuple((C1,C2),2)   :- constant(C1), constant(C2).
    argtuple((C1,V2),2)   :- constant(C1), V2 = 1..max_action_arity.
    argtuple((V1,C2),2)   :- V1 = 1..max_action_arity, constant(C2).
    argtuple((V1,V2),2)   :- V1 = 1..max_action_arity, V2 = 1..max_action_arity.

    \textcolor{bcolor}{% Tuples of objects that ground the action schemas and atoms}
    objtuple(I,  null_arg,0)      :- instance(I). \textcolor{bcolor}{% Explicit zero arity predicates}
    objtuple(I,     (O1,),1)      :- object(I,O1), not constant(O1).
    objtuple(I,   (O1,O2),2)      :- object(I,O1), object(I,O2), not constant(O1), not constant(O2), O1 $\neq$ O2.
    objtuple(I,   (O1,O1),2)      :- object(I,O1), not constant(O1),                                 opt_equal_objects = 1.

    \textcolor{bcolor}{% Tuples of objects/constants that appear as arguments to atoms}
    const_or_obj(I,O)             :- object(I,O).
    const_or_obj(I,O)             :- instance(I), constant(O).
    constobjtuple(I,  null_arg,0) :- instance(I). \textcolor{bcolor}{% Explicit zero arity predicates}
    constobjtuple(I,     (O1,),1) :- const_or_obj(I,O1).
    constobjtuple(I,   (O1,O2),2) :- const_or_obj(I,O1), const_or_obj(I,O2), O1 $\neq$ O2.
    constobjtuple(I,   (O1,O1),2) :- const_or_obj(I,O1),                     opt_equal_objects = 1.

    \textcolor{bcolor}{% Assumption: predicates have arity < 3}
    :- p_arity(F,N), N > 2.

    \textcolor{bcolor}{% Assert missing values for atoms (if some atom is not true, it is false)}
    fval(I,(F,OO),0)   :- feature(F),     f_static(I,F), p_arity(F,N), constobjtuple(I,OO,N),            not fval(I,(F,OO),1),   opt_fill = 1.
    fval(I,(F,OO),S,0) :- feature(F), not f_static(I,F), p_arity(F,N), constobjtuple(I,OO,N), node(I,S), not fval(I,(F,OO),S,1), opt_fill = 1.

    \textcolor{bcolor}{% Make sure we have full valuation of atoms}
    :-     f_static(I,F), p_arity(F,N), constobjtuple(I,OO,N),            \{ fval(I,(F,OO),0..1)   \} $\neq$ 1, opt_fill = 1.
    :- not f_static(I,F), p_arity(F,N), constobjtuple(I,OO,N), node(I,S), \{ fval(I,(F,OO),S,0..1) \} $\neq$ 1, opt_fill = 1.

    \textcolor{bcolor}{% Mapping of lifted arguments for atoms into grounded arguments. Lifted atom is pair (P,T) where P is}
    \textcolor{bcolor}{% predicate and T is tuple of variables and constants used to construct argument OO of grounded atom (P,OO).}

    \textcolor{bcolor}{% for nullary actions}
    map(I,(0,0,0,0),null_arg,null_arg,0)      :- instance(I).
    map(I,(0,0,0,0),(C,),(C,),1)              :- constobjtuple(I,(C,),1), constant(C).
    map(I,(0,0,0,0),(C1,C2),(C1,C2),2)        :- constobjtuple(I,(C1,C2),2), constant(C1), constant(C2).

    \textcolor{bcolor}{% for unary actions}
    map(I,(O1,0,0,0),null_arg,null_arg,0)     :- objtuple(I,(O1,),1).
    map(I,(O1,0,0,0),(C,),(C,),1)             :- objtuple(I,(O1,),1), constobjtuple(I,(C,),1), constant(C).
    map(I,(O1,0,0,0),(1,),(O1,),1)            :- objtuple(I,(O1,),1).
    map(I,(O1,0,0,0),(C1,C2),(C1,C2),2)       :- objtuple(I,(O1,),1), constobjtuple(I,(C1,C2),2), constant(C1), constant(C2).
    map(I,(O1,0,0,0),(C,1),(C,O1),2)          :- objtuple(I,(O1,),1), constobjtuple(I,(C,O1),2), constant(C).
    map(I,(O1,0,0,0),(1,C),(O1,C),2)          :- objtuple(I,(O1,),1), constobjtuple(I,(O1,C),2), constant(C).

    \textcolor{bcolor}{% for actions of arity >= 2}
    map(I,(O1,O2,0,0),null_arg,null_arg,0)    :- objtuple(I,(O1,O2),2).
    map(I,(O1,O2,0,0),(C,),(C,),1)            :- objtuple(I,(O1,O2),2), constobjtuple(I,(C,),1), constant(C).
    map(I,(O1,O2,0,0),(1,),(O1,),1)           :- objtuple(I,(O1,O2),2).
    map(I,(O1,O2,0,0),(2,),(O2,),1)           :- objtuple(I,(O1,O2),2).
    map(I,(O1,O2,0,0),(C1,C2),(C1,C2),2)      :- objtuple(I,(O1,O2),2), constobjtuple(I,(C1,C2),2), constant(C1), constant(C2).
    map(I,(O1,O2,0,0),(C,1),(C,O1),2)         :- objtuple(I,(O1,O2),2), constobjtuple(I,(C,O1),2), constant(C).
    map(I,(O1,O2,0,0),(C,2),(C,O2),2)         :- objtuple(I,(O1,O2),2), constobjtuple(I,(C,O2),2), constant(C).
    map(I,(O1,O2,0,0),(1,C),(O1,C),2)         :- objtuple(I,(O1,O2),2), constobjtuple(I,(O1,C),2), constant(C).
    map(I,(O1,O2,0,0),(2,C),(O2,C),2)         :- objtuple(I,(O1,O2),2), constobjtuple(I,(O2,C),2), constant(C).
    map(I,(O1,O2,0,0),(1,1),(O1,O1),2)        :- objtuple(I,(O1,O2),2), constobjtuple(I,(O1,O1),2).
    map(I,(O1,O2,0,0),(1,2),(O1,O2),2)        :- objtuple(I,(O1,O2),2), constobjtuple(I,(O1,O2),2).
  \end{Verbatim}
  \vskip -1eM
  \caption{Listing of the ASP code (page 1/4)}
  \label{fig:program:1}
\end{figure*}

\begin{figure*}[t]
  \begin{Verbatim}[gobble=4,frame=single,numbers=left,codes={\catcode`$=3},fontsize=\relsize{-2},firstnumber=97]
    map(I,(O1,O2,0,0),(2,1),(O2,O1),2)        :- objtuple(I,(O1,O2),2), constobjtuple(I,(O2,O1),2).
    map(I,(O1,O2,0,0),(2,2),(O2,O2),2)        :- objtuple(I,(O1,O2),2), constobjtuple(I,(O2,O2),2).
    map(I,(O1,0,O3,0),(3,),(O3,),1)           :- objtuple(I,(O1,O3),2).
    map(I,(O1,0,O3,0),(C,3),(C,O3),2)         :- objtuple(I,(O1,O3),2), constobjtuple(I,(C,O3),2), constant(C).
    map(I,(O1,0,O3,0),(3,C),(O3,C),2)         :- objtuple(I,(O1,O3),2), constobjtuple(I,(O3,C),2), constant(C).
    map(I,(O1,0,O3,0),(1,3),(O1,O3),2)        :- objtuple(I,(O1,O3),2), constobjtuple(I,(O1,O3),2).
    map(I,(O1,0,O3,0),(3,1),(O3,O1),2)        :- objtuple(I,(O1,O3),2), constobjtuple(I,(O3,O1),2).
    map(I,(O1,0,O3,0),(3,3),(O3,O3),2)        :- objtuple(I,(O1,O3),2), constobjtuple(I,(O3,O3),2).
    map(I,(O1,0,0,O4),(4,),(O4,),1)           :- objtuple(I,(O1,O4),2).
    map(I,(O1,0,0,O4),(C,4),(C,O4),2)         :- objtuple(I,(O1,O4),2), constobjtuple(I,(C,O4),2), constant(C).
    map(I,(O1,0,0,O4),(4,C),(O4,C),2)         :- objtuple(I,(O1,O4),2), constobjtuple(I,(O4,C),2), constant(C).
    map(I,(O1,0,0,O4),(1,4),(O1,O4),2)        :- objtuple(I,(O1,O4),2), constobjtuple(I,(O1,O4),2).
    map(I,(O1,0,0,O4),(4,1),(O4,O1),2)        :- objtuple(I,(O1,O4),2), constobjtuple(I,(O4,O1),2).
    map(I,(O1,0,0,O4),(4,4),(O4,O4),2)        :- objtuple(I,(O1,O4),2), constobjtuple(I,(O4,O4),2).
    map(I,(0,O2,O3,0),(2,3),(O2,O3),2)        :- objtuple(I,(O2,O3),2), constobjtuple(I,(O2,O3),2).
    map(I,(0,O2,O3,0),(3,2),(O3,O2),2)        :- objtuple(I,(O2,O3),2), constobjtuple(I,(O3,O2),2).
    map(I,(0,O2,0,O4),(2,4),(O2,O4),2)        :- objtuple(I,(O2,O4),2), constobjtuple(I,(O2,O4),2).
    map(I,(0,O2,0,O4),(4,2),(O4,O2),2)        :- objtuple(I,(O2,O4),2), constobjtuple(I,(O4,O2),2).
    map(I,(0,0,O3,O4),(3,4),(O3,O4),2)        :- objtuple(I,(O3,O4),2), constobjtuple(I,(O3,O4),2).
    map(I,(0,0,O3,O4),(4,3),(O4,O3),2)        :- objtuple(I,(O3,O4),2), constobjtuple(I,(O4,O3),2).

    map(I,(0,O2,0,0),(2,),(O2,),1)            :- objtuple(I,(O2,),1).
    map(I,(0,0,O3,0),(3,),(O3,),1)            :- objtuple(I,(O3,),1).
    map(I,(0,0,0,O4),(4,),(O4,),1)            :- objtuple(I,(O4,),1).


    \textcolor{bcolor}{% Define variables used by actions, vars in tuples, and good arg tuples for actions}
    a_var(A,V)            :- a_arity(A,N), V = 1..N.
    t_var((V,),V)         :- argtuple((V,),1),    not constant(V), V $\neq$ null.
    t_var((V1,V2),V1)     :- argtuple((V1,V2),2), not constant(V1).
    t_var((V1,V2),V2)     :- argtuple((V1,V2),2), not constant(V2).
    goodtuple(A,null_arg) :- action(A).
    goodtuple(A,T)        :- action(A), argtuple(T,N), a_var(A,V) : t_var(T,V).

    \textcolor{bcolor}{% Choice of lifted preconditions: prec(A,M,V) means atom M is prec of action A for V = 0 or 1}
    \{ prec(A,(P,T),1) \}      :- action(A), pred(P), p_arity(P,N), argtuple(T,N), goodtuple(A,T),
                                opt_allow_negative_precs = 0.

    \{ prec(A,(P,T),0..1) \} 1 :- action(A), pred(P), p_arity(P,N), argtuple(T,N), goodtuple(A,T),
                                opt_allow_negative_precs = 1.


    \textcolor{bcolor}{% Choice of lifted effects: eff(A,M,V) means atom M is effect of action A for V = 0 or 1}
    p_static(P) :- pred(P), f_static(I,P) : instance(I).
    \{ eff(A,(P,T),0..1) \} 1 :- action(A), pred(P), not p_static(P), p_arity(P,N), argtuple(T,N), goodtuple(A,T).

    \textcolor{bcolor}{% E1. Avoid noop actions and rule out contradictory effects}
    :- action(A), \{ eff(A,(P,T),0..1) : pred(P), p_arity(P,N), argtuple(T,N) \} = 0.
    :- eff(A,M,0), eff(A,M,1).

    \textcolor{bcolor}{% Factored applicability relations (this implementation for action arities up to 4)}

    \textcolor{bcolor}{% Static predicates}
    fappl(I,A,null_arg,null_arg)   :- instance(I), action(A), a_arity(A,0),
                                      fval(I,(P,OO),V)   : prec(A,(P,T),V), map(I,(0,0,0,0),T,OO,K),       f_static(I,P).
    fappl(I,A,(1,),(O1,))          :- instance(I), action(A), a_arity(A,1), objtuple(I,(O1,),1),
                                      fval(I,(P,OO),V)   : prec(A,(P,T),V), map(I,(O1,0,0,0),T,OO,K),      f_static(I,P).
    fappl(I,A,(1,2),(O1,O2))       :- instance(I), action(A), a_arity(A,N), N >= 2, objtuple(I,(O1,O2),2),
                                      fval(I,(P,OO),V)   : prec(A,(P,T),V), map(I,(O1,O2,0,0),T,OO,K),     f_static(I,P).
    fappl(I,A,(1,3),(O1,O3))       :- instance(I), action(A), a_arity(A,N), N >= 3, objtuple(I,(O1,O3),2),
                                      fval(I,(P,OO),V)   : prec(A,(P,T),V), map(I,(O1,0,O3,0),T,OO,K),     f_static(I,P).
    fappl(I,A,(2,3),(O2,O3))       :- instance(I), action(A), a_arity(A,N), N >= 3, objtuple(I,(O2,O3),2),
                                      fval(I,(P,OO),V)   : prec(A,(P,T),V), map(I,(0,O2,O3,0),T,OO,K),     f_static(I,P).
    fappl(I,A,(1,4),(O1,O4))       :- instance(I), action(A), a_arity(A,N), N >= 4, objtuple(I,(O1,O4),2),
                                      fval(I,(P,OO),V)   : prec(A,(P,T),V), map(I,(O1,0,0,O4),T,OO,K),     f_static(I,P).
    fappl(I,A,(2,4),(O2,O4))       :- instance(I), action(A), a_arity(A,N), N >= 4, objtuple(I,(O2,O4),2),
                                      fval(I,(P,OO),V)   : prec(A,(P,T),V), map(I,(0,O2,0,O4),T,OO,K),     f_static(I,P).
    fappl(I,A,(3,4),(O3,O4))       :- instance(I), action(A), a_arity(A,N), N >= 4, objtuple(I,(O3,O4),2),
                                      fval(I,(P,OO),V)   : prec(A,(P,T),V), map(I,(0,0,O3,O4),T,OO,K),     f_static(I,P).

    \textcolor{bcolor}{% Dynamic predicates (value depends on states)}
    fappl(I,A,null_arg,null_arg,S) :- instance(I), action(A), a_arity(A,0), relevant(I,S),
                                      fappl(I,A,null_arg,null_arg),
                                      fval(I,(P,OO),S,V) : prec(A,(P,T),V), map(I,(0,0,0,0),T,OO,K),   not f_static(I,P).
    fappl(I,A,(1,),(O1,),S)        :- instance(I), action(A), a_arity(A,1), objtuple(I,(O1,),1), relevant(I,S),
                                      fappl(I,A,(1,),(O1,)),
                                      fval(I,(P,OO),S,V) : prec(A,(P,T),V), map(I,(O1,0,0,0),T,OO,K),  not f_static(I,P).
    fappl(I,A,(1,2),(O1,O2),S)     :- instance(I), action(A), a_arity(A,N), N >= 2, objtuple(I,(O1,O2),2), relevant(I,S),
                                      fappl(I,A,(1,2),(O1,O2)),
                                      fval(I,(P,OO),S,V) : prec(A,(P,T),V), map(I,(O1,O2,0,0),T,OO,K), not f_static(I,P).
    fappl(I,A,(1,3),(O1,O3),S)     :- instance(I), action(A), a_arity(A,N), N >= 3, objtuple(I,(O1,O3),2), relevant(I,S),
                                      fappl(I,A,(1,3),(O1,O3)),
                                      fval(I,(P,OO),S,V) : prec(A,(P,T),V), map(I,(O1,0,O3,0),T,OO,K), not f_static(I,P).
    fappl(I,A,(2,3),(O2,O3),S)     :- instance(I), action(A), a_arity(A,N), N >= 3, objtuple(I,(O2,O3),2), relevant(I,S),
                                      fappl(I,A,(2,3),(O2,O3)),
                                      fval(I,(P,OO),S,V) : prec(A,(P,T),V), map(I,(0,O2,O3,0),T,OO,K), not f_static(I,P).
    fappl(I,A,(1,4),(O1,O4),S)     :- instance(I), action(A), a_arity(A,N), N >= 4, objtuple(I,(O1,O4),2), relevant(I,S),
                                      fappl(I,A,(1,4),(O1,O4)),
                                      fval(I,(P,OO),S,V) : prec(A,(P,T),V), map(I,(O1,0,0,O4),T,OO,K), not f_static(I,P).
    fappl(I,A,(2,4),(O2,O4),S)     :- instance(I), action(A), a_arity(A,N), N >= 4, objtuple(I,(O2,O4),2), relevant(I,S),
                                      fappl(I,A,(2,4),(O2,O4)),
                                      fval(I,(P,OO),S,V) : prec(A,(P,T),V), map(I,(0,O2,0,O4),T,OO,K), not f_static(I,P).
    fappl(I,A,(3,4),(O3,O4),S)     :- instance(I), action(A), a_arity(A,N), N >= 4, objtuple(I,(O3,O4),2), relevant(I,S),
                                      fappl(I,A,(3,4),(O3,O4)),
                                      fval(I,(P,OO),S,V) : prec(A,(P,T),V), map(I,(0,0,O3,O4),T,OO,K), not f_static(I,P).

  \end{Verbatim}
  \vskip -1eM
  \caption{Listing of the ASP code (page 2/4)}
  \label{fig:program:2}
\end{figure*}

\begin{figure*}[t]
  \begin{Verbatim}[gobble=4,frame=single,numbers=left,codes={\catcode`$=3},fontsize=\relsize{-2},firstnumber=193]
    \textcolor{bcolor}{% Factored next relation in the induced transition system}

    \textcolor{bcolor}{% Assumption in this implementation: each edge (S1,S2) labeled with unique action, and maps to unique grounded action A(OO)}
    :- tlabel(I,(S1,S2),A), tlabel(I,(S1,S2),B), A $\neq$ B.

    \textcolor{bcolor}{% fnext(I,K,O,S1,S2) : for edge tlabel(I,(S1,S2),A), K-th argument of A(OO) is object O}

    \textcolor{bcolor}{% 1. Account for proper def of next(I,A,OO,S1,S2):}
    \textcolor{bcolor}{%   1.a. Choose next S2 if ground action A(OO) is applicable in S1}
    \textcolor{bcolor}{%   1.b. If A/k1 -> O1 in (S1,S2), then A/k2 -> O2, for some O2, for each arg k1 of action A in (S1,S2)}
    \textcolor{bcolor}{%        (i.e., A(OO) must be fully defined for (S1,S2))}
    \textcolor{bcolor}{%   1.c. It cannot be fnext(I,K,O1,S1,S2) and fnext(I,K,O2,S1,S2) with O1 $\neq$ O2}
    \textcolor{bcolor}{%   1.d. A(OO) cannot be mapped to two edges (S1,S2) and (S1,S3) with S2 $\neq$ S3}
    \textcolor{bcolor}{%   1.e. If tlabel(I,(S1,S2),A), then fnext(I,0,0,S1,S2) or fnext(I,1,_,S1,S2)}

    \textcolor{bcolor}{% 1.a. Choose next S2 if ground action A(OO) is applicable in S1}
    \{ fnext(I,0,0,  S1,S2) : tlabel(I,(S1,S2),A) \} =1 :- action(A), a_arity(A,0), relevant(I,S1), fappl(I,A,null_arg,null_arg,S1).
    \{ fnext(I,1,O1, S1,S2) : tlabel(I,(S1,S2),A) \} =1 :- action(A), a_arity(A,1), relevant(I,S1), fappl(I,A,(1,),(O1,),S1).
    1 \{ fnext(I,1,O1, S1,S2) : tlabel(I,(S1,S2),A) \}  :- action(A), a_arity(A,2), relevant(I,S1), fappl(I,A,(1,2),(O1,O2),S1).
    1 \{ fnext(I,2,O2, S1,S2) : tlabel(I,(S1,S2),A) \}  :- action(A), a_arity(A,2), relevant(I,S1), fappl(I,A,(1,2),(O1,O2),S1).

    1 \{ fnext(I,1,O1, S1,S2) : tlabel(I,(S1,S2),A) \}  :- action(A), a_arity(A,3), relevant(I,S1),
                                                         fappl(I,A,(1,2),(O1,O2),S1), fappl(I,A,(1,3),(O1,O3),S1), fappl(I,A,(2,3),(O2,O3),S1).
    1 \{ fnext(I,2,O2, S1,S2) : tlabel(I,(S1,S2),A) \}  :- action(A), a_arity(A,3), relevant(I,S1),
                                                         fappl(I,A,(1,2),(O1,O2),S1), fappl(I,A,(1,3),(O1,O3),S1), fappl(I,A,(2,3),(O2,O3),S1).
    1 \{ fnext(I,3,O3, S1,S2) : tlabel(I,(S1,S2),A) \}  :- action(A), a_arity(A,3), relevant(I,S1),
                                                         fappl(I,A,(1,2),(O1,O2),S1), fappl(I,A,(1,3),(O1,O3),S1), fappl(I,A,(2,3),(O2,O3),S1).

    1 \{ fnext(I,1,O1, S1,S2) : tlabel(I,(S1,S2),A) \}  :- action(A), a_arity(A,4), relevant(I,S1),
                                                         fappl(I,A,(1,2),(O1,O2),S1), fappl(I,A,(1,3),(O1,O3),S1), fappl(I,A,(1,4),(O1,O4),S1),
                                                         fappl(I,A,(2,3),(O2,O3),S1), fappl(I,A,(2,4),(O2,O4),S1), fappl(I,A,(3,4),(O3,O4),S1).
    1 \{ fnext(I,2,O2, S1,S2) : tlabel(I,(S1,S2),A) \}  :- action(A), a_arity(A,4), relevant(I,S1),
                                                         fappl(I,A,(1,2),(O1,O2),S1), fappl(I,A,(1,3),(O1,O3),S1), fappl(I,A,(1,4),(O1,O4),S1),
                                                         fappl(I,A,(2,3),(O2,O3),S1), fappl(I,A,(2,4),(O2,O4),S1), fappl(I,A,(3,4),(O3,O4),S1).
    1 \{ fnext(I,3,O3, S1,S2) : tlabel(I,(S1,S2),A) \}  :- action(A), a_arity(A,4), relevant(I,S1),
                                                         fappl(I,A,(1,2),(O1,O2),S1), fappl(I,A,(1,3),(O1,O3),S1), fappl(I,A,(1,4),(O1,O4),S1),
                                                         fappl(I,A,(2,3),(O2,O3),S1), fappl(I,A,(2,4),(O2,O4),S1), fappl(I,A,(3,4),(O3,O4),S1).
    1 \{ fnext(I,4,O4, S1,S2) : tlabel(I,(S1,S2),A) \}  :- action(A), a_arity(A,4), relevant(I,S1),
                                                         fappl(I,A,(1,2),(O1,O2),S1), fappl(I,A,(1,3),(O1,O3),S1), fappl(I,A,(1,4),(O1,O4),S1),
                                                         fappl(I,A,(2,3),(O2,O3),S1), fappl(I,A,(2,4),(O2,O4),S1), fappl(I,A,(3,4),(O3,O4),S1).

    \textcolor{bcolor}{% 1.b. If k1 -> O1 in (S1,S2), then k2 -> O2 in (S1,S2), for some O2, for each arg k1 of action A}
    :- relevant(I,S1), tlabel(I,(S1,S2),A), a_arity(A,N), fnext(I,K1,O1,S1,S2), K1 > 0, K2 = 1..N, K2 $\neq$ K1,
       not fnext(I,K2,O2,S1,S2) : object(I,O2).

    \textcolor{bcolor}{% 1.c. It cannot be fnext(I,K,O1,S1,S2) and fnext(I,K,O2,S1,S2) with O1 < O2}
    :- relevant(I,S1), fnext(I,K,O1,S1,S2), fnext(I,K,O2,S1,S2), O1 < O2.

    \textcolor{bcolor}{% 1.d. A(OO) cannot be mapped to two edges (S1,S2) and (S1,S3) with S2 < S3}
    :- relevant(I,S1), tlabel(I,(S1,S2),A), tlabel(I,(S1,S3),A), S2 < S3, a_arity(A,0).
    :- relevant(I,S1), tlabel(I,(S1,S2),A), tlabel(I,(S1,S3),A), S2 < S3, a_arity(A,N), N >= 1,
       fnext(I,1,O,S1,S2), fnext(I,1,O,S1,S3), not diff_fnext(I,A,K2,S1,S2,S3) : K2 = 2..N.
    diff_fnext(I,A,K,S1,S2,S3) :- relevant(I,S1), tlabel(I,(S1,S2),A), tlabel(I,(S1,S3),A), S2 < S3,
                                  a_arity(A,N), N >= 1, K = 2..N, fnext(I,K,O1,S1,S2), fnext(I,K,O2,S1,S3), O1 $\neq$ O2.

    \textcolor{bcolor}{% 1.e. If tlabel(I,(S1,S2),A), then fnext(I,0,0,S1,S2) or fnext(I,1,_,S1,S2)}
    :- relevant(I,S1), tlabel(I,(S1,S2),A), not fnext(I,0,0,S1,S2), not fnext(I,1,O,S1,S2) : object(I,O).

    \textcolor{bcolor}{% 2. Check application of effects}
    \textcolor{bcolor}{%   2.a. If A(OO1) mapped to (S1,S2) and eff(A,(P,OO2),V), then fval(I,(P,OO2),S2,V)}
    \textcolor{bcolor}{%   2.b. If A(OO1) mapped to (S1,S2), fval(I,(P,OO2),S1,V) and fval(I,(P,OO2),S2,1-V), then eff(A,(P,OO2),1-V).}

    \textcolor{bcolor}{% 2.a. If A(OO1) mapped to (S1,S2) and eff(A,(P,OO2),V), then fval(I,(P,OO2),S2,V)}
    :- relevant(I,S1), tlabel(I,(S1,S2),A), a_arity(A,0),         fnext(I,0,0,S1,S2),                       eff(A,(P,T),V),
       map(I,(0,0,0,0),T,OO,K),   not f_static(I,P), fval(I,(P,OO),S2,1-V).
    :- relevant(I,S1), tlabel(I,(S1,S2),A), a_arity(A,1),         fnext(I,1,O1,S1,S2),                      eff(A,(P,T),V),
       map(I,(O1,0,0,0),T,OO,K),  not f_static(I,P), fval(I,(P,OO),S2,1-V).
    :- relevant(I,S1), tlabel(I,(S1,S2),A), a_arity(A,N), N >= 2, fnext(I,1,O1,S1,S2), fnext(I,2,O2,S1,S2), eff(A,(P,T),V),
       map(I,(O1,O2,0,0),T,OO,K), not f_static(I,P), fval(I,(P,OO),S2,1-V).
    :- relevant(I,S1), tlabel(I,(S1,S2),A), a_arity(A,N), N >= 3, fnext(I,1,O1,S1,S2), fnext(I,3,O3,S1,S2), eff(A,(P,T),V),
       map(I,(O1,0,O3,0),T,OO,K), not f_static(I,P), fval(I,(P,OO),S2,1-V).
    :- relevant(I,S1), tlabel(I,(S1,S2),A), a_arity(A,N), N >= 3, fnext(I,2,O2,S1,S2), fnext(I,3,O3,S1,S2), eff(A,(P,T),V),
       map(I,(0,O2,O3,0),T,OO,K), not f_static(I,P), fval(I,(P,OO),S2,1-V).
    :- relevant(I,S1), tlabel(I,(S1,S2),A), a_arity(A,N), N >= 4, fnext(I,1,O1,S1,S2), fnext(I,4,O4,S1,S2), eff(A,(P,T),V),
       map(I,(O1,0,0,O4),T,OO,K), not f_static(I,P), fval(I,(P,OO),S2,1-V).
    :- relevant(I,S1), tlabel(I,(S1,S2),A), a_arity(A,N), N >= 4, fnext(I,2,O2,S1,S2), fnext(I,4,O4,S1,S2), eff(A,(P,T),V),
       map(I,(0,O2,0,O4),T,OO,K), not f_static(I,P), fval(I,(P,OO),S2,1-V).
    :- relevant(I,S1), tlabel(I,(S1,S2),A), a_arity(A,N), N >= 4, fnext(I,3,O3,S1,S2), fnext(I,4,O4,S1,S2), eff(A,(P,T),V),
       map(I,(0,0,O3,O4),T,OO,K), not f_static(I,P), fval(I,(P,OO),S2,1-V).

    :- relevant(I,S1), tlabel(I,(S1,S2),A), a_arity(A,0),         fnext(I,0,0,S1,S2),                       eff(A,(P,T),V),
       map(I,(0,0,0,0),T,OO,K),       f_static(I,P), fval(I,(P,OO),1-V).
    :- relevant(I,S1), tlabel(I,(S1,S2),A), a_arity(A,1),         fnext(I,1,O1,S1,S2),                      eff(A,(P,T),V),
       map(I,(O1,0,0,0),T,OO,K),      f_static(I,P), fval(I,(P,OO),1-V).
    :- relevant(I,S1), tlabel(I,(S1,S2),A), a_arity(A,N), N >= 2, fnext(I,1,O1,S1,S2), fnext(I,2,O2,S1,S2), eff(A,(P,T),V),
       map(I,(O1,O2,0,0),T,OO,K),     f_static(I,P), fval(I,(P,OO),1-V).
    :- relevant(I,S1), tlabel(I,(S1,S2),A), a_arity(A,N), N >= 3, fnext(I,1,O1,S1,S2), fnext(I,3,O3,S1,S2), eff(A,(P,T),V),
       map(I,(O1,0,O3,0),T,OO,K),     f_static(I,P), fval(I,(P,OO),1-V).
    :- relevant(I,S1), tlabel(I,(S1,S2),A), a_arity(A,N), N >= 3, fnext(I,2,O2,S1,S2), fnext(I,3,O3,S1,S2), eff(A,(P,T),V),
       map(I,(0,O2,O3,0),T,OO,K),     f_static(I,P), fval(I,(P,OO),1-V).
    :- relevant(I,S1), tlabel(I,(S1,S2),A), a_arity(A,N), N >= 4, fnext(I,1,O1,S1,S2), fnext(I,4,O4,S1,S2), eff(A,(P,T),V),
       map(I,(O1,0,0,O4),T,OO,K),     f_static(I,P), fval(I,(P,OO),1-V).
    :- relevant(I,S1), tlabel(I,(S1,S2),A), a_arity(A,N), N >= 4, fnext(I,2,O2,S1,S2), fnext(I,4,O4,S1,S2), eff(A,(P,T),V),
       map(I,(0,O2,0,O4),T,OO,K),     f_static(I,P), fval(I,(P,OO),1-V).
    :- relevant(I,S1), tlabel(I,(S1,S2),A), a_arity(A,N), N >= 4, fnext(I,3,O3,S1,S2), fnext(I,4,O4,S1,S2), eff(A,(P,T),V),
       map(I,(0,0,O3,O4),T,OO,K),     f_static(I,P), fval(I,(P,OO),1-V).
  \end{Verbatim}
  \vskip -1eM
  \caption{Listing of the ASP code (page 3/4)}
  \label{fig:program:3}
\end{figure*}

\begin{figure*}[t]
  \begin{Verbatim}[gobble=4,frame=single,numbers=left,codes={\catcode`$=3},fontsize=\relsize{-2},firstnumber=289]
    \textcolor{bcolor}{% 2.b. If A(OO1) mapped to (S1,S2), fval(I,(P,OO2),S1,V) and fval(I,(P,OO2),S2,1-V), then eff(A,(P,OO2),1-V).}
    :- relevant(I,S1), tlabel(I,(S1,S2),A),
       pred(P), fval(I,(P,null_arg),S1,V), fval(I,(P,null_arg),S2,1-V),
       not eff(A,(P,null_arg),1-V).

    :- relevant(I,S1), tlabel(I,(S1,S2),A), a_arity(A,N),
       pred(P), fval(I,(P,OO),S1,V), fval(I,(P,OO),S2,1-V), constobjtuple(I,OO,1),
       not eff(A,(P,T),1-V) : map(I,(0,0,0,0),T,OO,1),                        N >= 0;
       not eff(A,(P,T),1-V) : map(I,(O1,0,0,0),T,OO,1),  fnext(I,1,O1,S1,S2), N >= 1;
       not eff(A,(P,T),1-V) : map(I,(0,O2,0,0),T,OO,1),  fnext(I,2,O2,S1,S2), N >= 2;
       not eff(A,(P,T),1-V) : map(I,(0,0,O3,0),T,OO,1),  fnext(I,3,O3,S1,S2), N >= 3;
       not eff(A,(P,T),1-V) : map(I,(0,0,0,O4),T,OO,1),  fnext(I,4,O4,S1,S2), N >= 4.

    :- relevant(I,S1), tlabel(I,(S1,S2),A), a_arity(A,N),
       pred(P), fval(I,(P,OO),S1,V), fval(I,(P,OO),S2,1-V), constobjtuple(I,OO,2),
       not eff(A,(P,T),1-V) : map(I,(0,0,0,0),T,OO,2),                                             N >= 0;
       not eff(A,(P,T),1-V) : map(I,(O1,0,0,0),T,OO,2),  fnext(I,1,O1,S1,S2),                      N >= 1;
       not eff(A,(P,T),1-V) : map(I,(O1,O2,0,0),T,OO,2), fnext(I,1,O1,S1,S2), fnext(I,2,O2,S1,S2), N >= 2;
       not eff(A,(P,T),1-V) : map(I,(O1,0,O3,0),T,OO,2), fnext(I,1,O1,S1,S2), fnext(I,3,O3,S1,S2), N >= 3;
       not eff(A,(P,T),1-V) : map(I,(0,O2,O3,0),T,OO,2), fnext(I,2,O2,S1,S2), fnext(I,3,O3,S1,S2), N >= 3;
       not eff(A,(P,T),1-V) : map(I,(O1,0,0,O4),T,OO,2), fnext(I,1,O1,S1,S2), fnext(I,4,O4,S1,S2), N >= 4;
       not eff(A,(P,T),1-V) : map(I,(0,O2,0,O4),T,OO,2), fnext(I,2,O2,S1,S2), fnext(I,4,O4,S1,S2), N >= 4;
       not eff(A,(P,T),1-V) : map(I,(0,0,O3,O4),T,OO,2), fnext(I,3,O3,S1,S2), fnext(I,4,O4,S1,S2), N >= 4.

    \textcolor{bcolor}{% 3. Check application of applicable ground actions}
    :- a_arity(A,2), fappl(I,A,(1,2),(O1,O2),S1), not fnext(I,2,O2,S1,S2) : fnext(I,1,O1,S1,S2).

    :- a_arity(A,3), fappl(I,A,(1,2),(O1,O2),S1), fappl(I,A,(1,3),(O1,O3),S1), fappl(I,A,(2,3),(O2,O3),S1),
       not fnext(I,2,O2,S1,S2) : fnext(I,1,O1,S1,S2).
    :- a_arity(A,3), fappl(I,A,(1,2),(O1,O2),S1), fappl(I,A,(1,3),(O1,O3),S1), fappl(I,A,(2,3),(O2,O3),S1),
       not fnext(I,3,O3,S1,S2) : fnext(I,1,O1,S1,S2), fnext(I,2,O2,S1,S2).

    :- a_arity(A,4), fappl(I,A,(1,2),(O1,O2),S1), fappl(I,A,(1,3),(O1,O3),S1), fappl(I,A,(1,4),(O1,O4),S1),
       fappl(I,A,(2,3),(O2,O3),S1), fappl(I,A,(2,4),(O2,O4),S1), fappl(I,A,(3,4),(O3,O4),S1),
       not fnext(I,2,O2,S1,S2) : fnext(I,1,O1,S1,S2).
    :- a_arity(A,4), fappl(I,A,(1,2),(O1,O2),S1), fappl(I,A,(1,3),(O1,O3),S1), fappl(I,A,(1,4),(O1,O4),S1),
       fappl(I,A,(2,3),(O2,O3),S1), fappl(I,A,(2,4),(O2,O4),S1), fappl(I,A,(3,4),(O3,O4),S1),
       not fnext(I,3,O3,S1,S2) : fnext(I,1,O1,S1,S2), fnext(I,2,O2,S1,S2).
    :- a_arity(A,4), fappl(I,A,(1,2),(O1,O2),S1), fappl(I,A,(1,3),(O1,O3),S1), fappl(I,A,(1,4),(O1,O4),S1),
       fappl(I,A,(2,3),(O2,O3),S1), fappl(I,A,(2,4),(O2,O4),S1), fappl(I,A,(3,4),(O3,O4),S1),
       not fnext(I,4,O4,S1,S2) : fnext(I,1,O1,S1,S2), fnext(I,2,O2,S1,S2), fnext(I,3,O3,S1,S2).

    \textcolor{bcolor}{% Fundamental constraints}

    \textcolor{bcolor}{% C1. Different nodes are different states with respect to chosen predicates}
    :- relevant(I,S1), relevant(I,S2), S1 < S2, fval(I,(P,OO),S2,V) : fval(I,(P,OO),S1,V), pred(P), not f_static(I,P).

    \textcolor{bcolor}{% C2. If (S1,S2) has label A, then A(OO) maps S1 to S2 for some object tuple OO}
    :- relevant(I,S1), a_arity(A,0), tlabel(I,(S1,S2),A), not fnext(I,0,0,S1,S2).
    :- relevant(I,S1), a_arity(A,1), tlabel(I,(S1,S2),A), #count \{    O1 : fnext(I,1,O1,S1,S2) \} $\neq$ 1.
    :- relevant(I,S1), a_arity(A,2), tlabel(I,(S1,S2),A), #count \{ O1,O2 : fnext(I,1,O1,S1,S2), fnext(I,2,O2,S1,S2) \} $\neq$ 1.
    :- relevant(I,S1), a_arity(A,3), tlabel(I,(S1,S2),A),
       #count \{    O1,O2,O3 : fnext(I,1,O1,S1,S2), fnext(I,2,O2,S1,S2), fnext(I,3,O3,S1,S2) \} $\neq$ 1.
    :- relevant(I,S1), a_arity(A,4), tlabel(I,(S1,S2),A),
       #count \{ O1,O2,O3,O4 : fnext(I,1,O1,S1,S2), fnext(I,2,O2,S1,S2), fnext(I,3,O3,S1,S2), fnext(I,4,O4,S1,S2) \} $\neq$ 1.

    \textcolor{bcolor}{% Break symmetries (taken from STRIPS learner's ASP code)}
    a_atom(1,A,M) :- prec(A,M,V).
    a_atom(2,A,M) :-  eff(A,M,V).
    a_atom(A,M)   :- a_atom(I,A,M), I = 1..2.

    :- V = 1..max_action_arity-1, action(A),
       1 #sum\{ -1,I,P,VV,1 : a_atom(I,A,(P,(  V, VV)));
               -1,I,P,VV,2 : a_atom(I,A,(P,( VV,  V)));
                1,I,P,VV,1 : a_atom(I,A,(P,(V+1, VV)));
                1,I,P,VV,2 : a_atom(I,A,(P,( VV,V+1))) \},
       opt_symmetries = 1.

    \textcolor{bcolor}{% Optimization (lexicographic ordering)}
    \textcolor{bcolor}{% - Prefer models with minimum sum of actions' cardinalities}
    \textcolor{bcolor}{% - Prefer models with minimum sum of (non-static) predicates' cardinalities}
    \textcolor{bcolor}{% - Prefer models with minimum sum of (static) predicates' cardinalities}
    \textcolor{bcolor}{% - Prefer models with minimum number of effects}
    \textcolor{bcolor}{% - Prefer models with minimum number of preconditions}
    #minimize \{ 1+N@10, A : a_arity(A,N) \}.
    #minimize \{ 1+N@8, P : pred(P), p_arity(P,N), not p_static(P) \}.
    #minimize \{ 1+N@6, P : pred(P), p_arity(P,N), p_static(P) \}.
    #minimize \{ 1@4, A, P, T, V : eff(A,(P,T),V) \}.
    #minimize \{ 1@2, A, P, T, V : prec(A,(P,T),V) \}.

    \textcolor{bcolor}{% Default is to display nothing}
    #show.

    \textcolor{bcolor}{% Display objects and constants}
    #show object/2.
    #show constant/1.

    \textcolor{bcolor}{% Display selected predicates}
    #show pred/1.
    #show p_static(P) : p_static(P), pred(P).

    \textcolor{bcolor}{% Display action schemas}
    #show action/1.
    #show a_arity/2.
    #show prec/3.
    #show eff/3.
  \end{Verbatim}
  \vskip -1eM
  \caption{Listing of the ASP code (page 4/4)}
  \label{fig:program:4}
\end{figure*}

% Omit

\end{document}